\documentclass[runningheads]{llncs}

\usepackage{eccv}

\usepackage{eccvabbrv}
\usepackage[normalem]{ulem}
\usepackage{graphicx}
\usepackage{booktabs}
\usepackage{algorithm}
\usepackage{multirow}
\usepackage{algorithmic}
\usepackage{caption}
\usepackage{wrapfig}
\usepackage[accsupp]{axessibility}  % Improves PDF readability for those with disabilities.
\usepackage{orcidlink}

\usepackage{hyperref}

\usepackage{orcidlink}
\makeatletter
\def\blfootnote{\xdef\@thefnmark{}\@footnotetext}
\makeatother

\begin{document}

\title{$h$-Flow: Flexible Flow-based Image Editing via Doob's $h$-Transform} 
% 如果有作者信息，写在这里
% \author{...}

% Second Author\inst{2,3}\orcidlink{1111-2222-3333-4444} \and
% Third Author\inst{3}\orcidlink{2222--3333-4444-5555}}
% \titlerunning{Abbreviated paper title} % 如果有短标题放这里，没有可以注释掉

\author{Zehui Guo\inst{1}\orcidlink{0009-0005-9932-9440} \and
Zhen Wang\inst{2}\orcidlink{0009-0008-1091-7994} \and
Junwei Shu\inst{1}\orcidlink{0009-0006-1197-032X} \and
Yang Li\inst{1}\textsuperscript{,*}\orcidlink{0000-0001-9427-7665} \and
Changbo Wang\inst{1}\textsuperscript{,*}\orcidlink{0000-0001-8940-6418} \and
Long Chen\inst{2}\orcidlink{0000-0001-6148-9709}}

% 页眉处的作者缩写
\authorrunning{Z.~Guo et al.}

\institute{East China Normal University, Shanghai, China\\
\email{51275901098@stu.ecnu.edu.cn, 51265901091@stu.ecnu.edu.cn, yli@cs.ecnu.edu.cn, cbwang@cs.ecnu.edu.cn} \and
The Hong Kong University of Science and Technology, Hong Kong SAR, China\\
\email{zhenwang@ust.hk, longchen@ust.hk}} 

\maketitle

% Second Author\inst{2,3}\orcidlink{1111-2222-3333-4444} \and
% Third Author\inst{3}\orcidlink{2222--3333-4444-5555}}



% Springer Heidelberg, Tiergartenstr.~17, 69121 Heidelberg, Germany

% \maketitle

% \maketitle % <--- 关键：必须先生成标题

% 将图片放在 \maketitle 和 \begin{abstract} 之间
\begin{figure}[h!] % <--- 改为 [h!]
  \centering
  \includegraphics[width=1\linewidth]{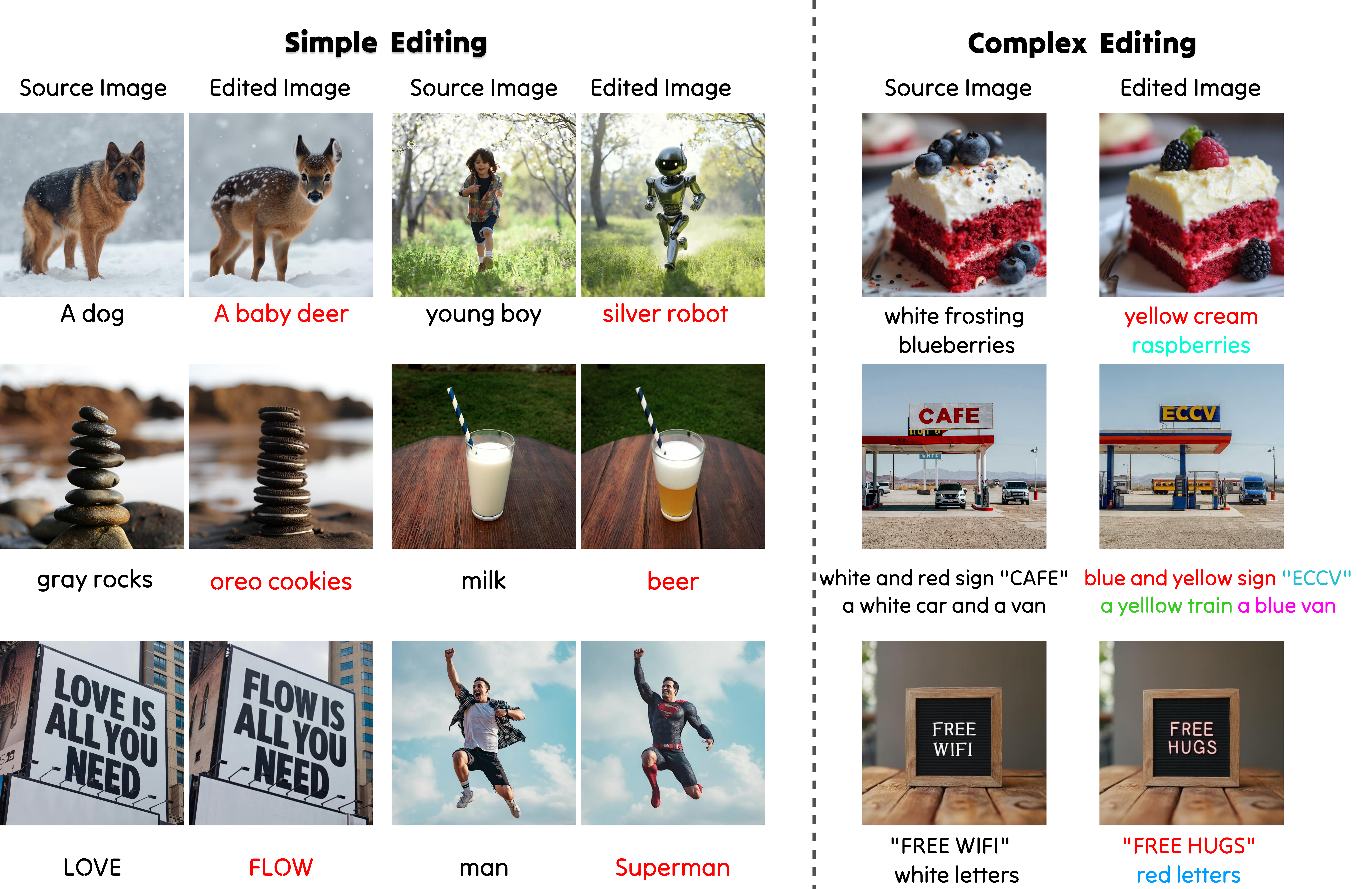}
  \caption{$h$-Flow enables faithful text-based editing across diverse scenarios.}
  \vspace{-2em}
  \label{fig:example}
\end{figure}

% $h$-Flow edits images by constructing a \emph{reverse-time bridge} over pre-trained Rectified Flow models. The method derives exact, closed-form reconstruction guidance from Doob's $h$-transform and decouples it from the editing signal via orthogonal projection, enabling faithful source preservation and effective semantic modification without training, gradients, or model-specific tuning. Text indicates the target edit; unmentioned regions are preserved automatically.
% \title{$h$-Flow: Flexible Flow-based Image Editing via Doob's $h$-Transform} 


% \maketitle

\newcommand{\zhen}[1]{{\color{orange}{$^\textbf{\emph{Zhen:}}$[#1]}}}

\begin{abstract}
\blfootnote{$^*$Corresponding author.}
Editing images with pre-trained text-to-image flow models typically requires carefully balancing target alignment with the desired prompt and source consistency with the original image. Existing approaches either rely on inversion-based pipelines or heuristic source-to-target trajectory constructions, which often depend on architecture-specific designs or are sensitive to hyperparameters. In this paper, we propose \textbf{h-flow}, a training-free and theoretically grounded flow-based editing framework. Inspired by Doob’s $h$-Transform, we reformulate image editing as conditional generation under multiple terminal events corresponding to source consistency and target alignment. We first extend the classical $h$-Transform from SDE-based models to the deterministic RF framework by constructing an equivalent SDE with identical marginals. Within this formulation, we design dedicated $h$-functions for source consistency and target alignment, yielding closed-form reconstruction guidance and velocity-based semantic editing signals. We further introduce a velocity orthogonal decomposition to decouple reconstruction and editing directions, enabling a controllable trade-off between the two objectives. Extensive experiments demonstrate that h-flow achieves effective, robust, and flexible editing across diverse scenarios. The code will be released soon.

\keywords{Text-based Image Editing \and Rectified Flow Model \and Doob's $h$-Transform}
\end{abstract}

\section{Introduction}
\label{sec:intro}

Rectified flow (RF) models~\cite{liuflow,lipman2022flow} have emerged as a powerful generative paradigm for various visual generation tasks. Benefiting from large-scale pretraining, text-to-image RF models~\cite{esser2024scaling, flux2024} have been widely adopted for text-based image editing~\cite{rout2025semantic,kulikov2025flowedit}, which aims to modify the visual contents of a source image according to a target prompt. The central challenge of this task lies in balancing the two common goals: \emph{target alignment} and \emph{source consistency}. Specifically, the synthesized image should faithfully adhere to the target prompt while preserving the editing-irrelevant regions unchanged from the source image.

To achieve this trade-off, early flow-based editing approaches predominantly adopt a two-stage, inversion-based paradigm~\cite{wang2025taming,rout2025semantic,kim2025reflex}. In the forward process, the source image is first inverted into a latent variable in the noise distribution, in the backward process, the edited image is generated from this latent under the guidance of the target prompt. Representative works mainly focus on improving inversion accuracy~\cite{rout2025semantic,kim2025reflex} for better source reconstruction, and further inject source conditions (\eg, attention maps~\cite{wang2025taming}) during the forward process to enhance source consistency. However, these strategies are largely heuristic and often rely on architecture-specific analysis, limiting their scalability and general applicability. More recent studies have pioneered inversion-free paradigms~\cite{kulikov2025flowedit,kim2025flowalign,jiang2026flowdc}, which directly construct transformation trajectories from the source image to the target image, enforcing strong structural consistency. Despite their empirical effectiveness, they remain sensitive to specific hyperparameter settings, such as random seeds and the number of flow steps, which constrain their robustness and generalization in practical scenarios.

Different from prior methods that are primarily motivated by empirical intuition without a unified theoretical grounding, we propose \textbf{\emph{h}-flow}, a training-free and flexible flow-based editing approach built upon the intrinsic probabilistic structure of flow models. Specifically, inspired by Doob’s $h$-Transform~\cite{doob1984classical,rogers2000diffusions,sarkka2019applied}, we reformulate image editing as a conditional generation problem under multiple terminal events, corresponding to the two fundamental objectives of editing: \emph{source consistency} with the input image and \emph{target alignment} with the target prompt. Since the classical $h$-Transform is defined for SDE-based stochastic processes, whereas RF models operate under a deterministic ODE formulation, we first construct an equivalent SDE that shares the same marginal distributions as the RF dynamics. This enables us to rigorously extend the $h$-Transform from diffusion/SDE-based models (\eg, h-Edit~\cite{nguyen2025h}) to the ODE-based RF framework, providing a principled theoretical bridge. Within this unified formulation, we explicitly design dedicated $h$-functions to encode the two core objectives of image editing. Specifically, the $h$-function associated with source consistency imposes a terminal constraint on the source image, which naturally induces an exact, closed-form reconstruction guidance along the flow trajectory. In contrast, the $h$-function for target alignment models the desired semantic transition toward the target prompt. Under this design, the editing signal is characterized by the deviation of the editing velocity from the source velocity, thereby explicitly capturing the semantic shift between the source and target prompts in the velocity field. 

To explicitly balance the two objectives, we further propose a velocity orthogonal decomposition strategy, which projects the editing velocity onto complementary subspaces. This design ensures that reconstruction and editing directions act independently, fully decoupling source preservation and target modification at the velocity level, and thus inherently achieving a controllable trade-off between source consistency and target alignment. Moreover, since $h$-flow modifies the backward process while remaining agnostic to the specific forward construction, our method can be seamlessly combined with different forward processes in a plug-and-play manner, further enhancing editing flexibility and performance. Extensive experiments demonstrate the effectiveness, robustness, and generalization capability of \textbf{\emph{h}-flow} across diverse editing scenarios.

In summary, our contributions can be summarized as follows:
\begin{itemize}
    \item We propose \textbf{\emph{h}-flow}, a training-free and theoretically grounded flow-based image editing framework that reformulates editing via Doob’s $h$-Transform.
    \item We design dedicated $h$-functions to explicitly encode source consistency and target alignment. Together with a velocity orthogonal decomposition strategy, our method inherently decouples reconstruction and editing to achieve a controllable trade-off, and can be flexibly combined with different forward processes in a plug-and-play manner.
    \item Extensive experiments and ablations demonstrate the effectiveness and robustness of \textbf{\emph{h}-flow} under diverse editing scenarios.
\end{itemize}

\section{Related work}

\noindent{\textbf{Inversion-based Editing}}
This paradigm first inverts the source image into a noise latent through a forward inversion process, and then re-generates the edited result under the target prompt in the backward process. To mitigate error accumulation during inversion and subsequent generation, a line of works~\cite{mokady2023null,wang2025taming,rout2025semantic,dao2026steerflow} focuses on improving inversion accuracy. For example, RF-Solver-Edit~\cite{wang2025taming} designs higher-order ODE solvers to reduce discretization errors, while RF-Inversion~\cite{rout2025semantic} formulates inversion as a dynamic optimal control problem to refine the trajectory. However, even accurate inversion cannot fully guarantee source fidelity after editing. Therefore, another line of works~\cite{hertzprompt,tumanyan2023plug,kim2025reflex,xu2024freetuner,wang2025taming} enhances source consistency by injecting additional source conditions during the backward generation process. Despite improved source preservation, these methods typically rely on heuristic designs and architecture-specific modifications, limiting their generality and flexibility.

\noindent{\textbf{Inversion-Free Editing.}}
A growing body of work seeks to bypass explicit inversion. SDEdit~\cite{mengsdedit} directly perturbs the source image with random noise to replace the inversion step. However, the noise strength introduces a hard trade-off between editability and source preservation. FlowEdit~\cite{kulikov2025flowedit} pioneers a dual-trajectory paradigm that constructs a direct source-to-target ODE by coupling two noise-to-image flows, thereby avoiding inversion. Nevertheless, per-step random noise sampling can cause trajectory instability. FlowAlign~\cite{kim2025flowalign} further improves stability within this framework by formulating source preservation as an optimal control problem and optimizing the trajectory accordingly. Despite these advances, such source-to-target constructions remain largely heuristic, lacking a unified theoretical foundation, and are sensitive to hyperparameter choices such as random seeds and flow steps.

\noindent{\textbf{{Doob's h-Transform in Image Editing.}} 
Following~\cite{liu2023learning,zhou2024denoising,kieu2025bidirectional,wang2026coarse}, we define bridges as stochastic processes conditioned on a predefined terminal event, derived from a base Markov process via Doob's $h$-transform~\cite{doob1984classical,rogers2000diffusions,sarkka2019applied}. Specifically, h-Edit~\cite{nguyen2025h} is the first to introduce the $h$-Transform into image editing. However, it is developed within the diffusion SDE framework and approximates the source constraint through a text prompt proxy, without enforcing an explicit hard constraint on the source image. In contrast, we extend the $h$-Transform to the deterministic flow matching framework, and derive an exact closed-form guidance for the source image constraint, resulting in a more principled, flexible, and controllable editing formulation.

\section{Preliminaries}

\subsection{Rectified Flow}

Rectified Flow (RF)~\cite {liuflow} defines a forward process as a linear interpolation between a data sample $x_0 \sim p_0$ and Gaussian noise $\epsilon \sim \mathcal{N}(0, I)$: 
\begin{equation}
    x_t = (1-t)x_0 + t\epsilon, \quad t \in [0, 1]
    \label{eq:rf_forward}
\end{equation}
 The conditional transition kernel is an isotropic Gaussian:
 \begin{equation}
    p(x_t \mid x_0) = \mathcal{N}\bigl(x_t;\,(1-t)x_0,\,t^2 I\bigr). 
    \label{eq:rf_kernel}
 \end{equation}
A velocity network $v_\theta(x_t, t)$ is trained via the flow matching objective~\cite{lipman2022flow} to approximate $\mathbb{E}[\epsilon - x_0 \mid x_t]$. Samples are generated by solving the ODE $dx_t/dt = v_\theta(x_t, t)$ backward from $t{=}1$ to $t{=}0$.

Applying the Tweedie formula to the RF kernel~\eqref{eq:rf_kernel} gives the posterior mean $\mathbb{E}[x_0 \mid x_t]$, which we approximate using the learned velocity:
 \begin{equation}
    \hat{x}_0 \triangleq x_t - t\,v_\theta(x_t, t) \approx \mathbb{E}[x_0 \mid x_t].
    \label{eq:tweedie}
 \end{equation}
The approximation is exact when $v_\theta$ matches the true conditional expectation $\mathbb{E}[\epsilon - x_0 \mid x_t]$.
From the Gaussian form of Eq.~\eqref{eq:rf_kernel} and Eq.~\eqref{eq:tweedie}, one can derive the score--velocity identity:
\begin{equation}
\nabla_{x_t} \log p_t(x_t) = \frac{(1-t)\,\hat{x}_0 - x_t}{t^2},
\label{eq:score_identity}
\end{equation}
which will be used repeatedly in our derivations.

\subsection{Inversion-based Image Editing}
\label{sec:prelim_editing}

Given a source image $I^{src}$, a source prompt $c^{src}$, and a target prompt $c^{edit}$, text-based image editing aims to produce an output $I^{edit}$ that (i) reflects the semantic content described by $c^{edit}$ (\emph{target alignment}), while (ii) preserving the structure and unedited regions of $I^{src}$ (\emph{source consistency}). Both images are encoded as $x_0^{src} = \mathcal{E}(I^{src})$ and $x_0^{edit} = \mathcal{E}(I^{edit})$ via a pretrained VAE encoder $\mathcal{E}$, and all operations are performed in this latent space.

Inversion-based editing typically consists of two stages. The forward process maps the source image to an intermediate latent state $x_{t_N}$ at some time $t_N \in (0,1]$, from which the backward generation begins. This can be achieved either through iterative inversion, $x_{t_N} = x_0^{src} + \int_{0}^{x_{t_N}} v_\theta(x_t^{src}, t, c^{src}) \, dt$ or via on-step noising  $x_{t_N} = (1{-}t_N)\,x_0^{src} + t_N\,\epsilon$. The backward process then generates the edited latent from this intermediate state under the target prompt by integrating the conditional flow in reverse time:
\begin{align}
\label{eq:edit}
x_0^{edit} &= x_{t_N} - \int_{0}^{{t_N}} {v}_\theta(x_t, t, c^{edit}) \, dt
\end{align}

\subsection{Doob's h-Transform}

Doob's h-Transform~\cite{doob1984classical,rogers2000diffusions,sarkka2019applied} is a classical tool for conditioning a Markov process on a future event. Consider an Itô SDE:
\begin{equation}
    dx = f(x_t, t)\,dt + g(t)\,dw.
\end{equation}
 To condition this process on a terminal event $A$ (e.g., $x_0 \in A$), the h-Transform introduces an additive drift correction proportional to the log-gradient of the harmonic function $h(x_t, t) = p_t(A \mid x_t)$: 
 \begin{equation}
      dx = \bigl[f(x_t, t) + g^2(t)\,\nabla_{x_t} \log p_t(A \mid x_t)\bigr]\,dt + g(t)\,dw.
      \label{eq:prelim_doob_h_SDE_form}
 \end{equation}
The corresponding probability-flow ODE (PF-ODE)~\cite{chen2023probability}, which shares the same marginal distributions, is:
\begin{equation}
\frac{dx}{dt} = \underbrace{f(x_t,t) - \frac{g^2(t)}{2}\,\nabla_{x_t}\!\log p_t(x_t)}_{\widetilde{f}\;(\text{unconditioned PF-ODE drift})} + \frac{g^2(t)}{2}\,\nabla_{x_t}\!\log p_t(A \mid x_t).
\label{eq:h_pfode_prelim}
\end{equation}
The first part $\widetilde{f}$ is the standard (unconditioned) PF-ODE drift; the second part is the bridge guidance with a \emph{positive} sign. When $A$ corresponds to a class label, Eq.~(\ref{eq:h_pfode_prelim}) recovers classifier guidance~\cite{dhariwal2021diffusion}; when combined with a null-condition baseline, it reduces to classifier-free guidance~\cite{ho2021classifier}. The $h$-transform thus provides a unified probabilistic foundation for guided generation.
 
\noindent\textbf{The $g > 0$ requirement.} The guidance terms in Eqs.~\eqref{eq:prelim_doob_h_SDE_form}--\eqref{eq:h_pfode_prelim} are scaled by $g^2(t)$ or $g^2(t)/2$. When $g(t) > 0$, the infinitesimal transition kernel has non-degenerate variance, and the $h$-transform reweights it smoothly. When $g \equiv 0$---as in
Rectified Flow---the transition kernel collapses to a Dirac delta, and the $h$-transform is undefined. Resolving this incompatibility is the central technical challenge addressed in Sec.~\ref{sec:method}. The $h$-transform is presented above in forward time for clarity. When the process runs in reverse time ($t{:}\,1{\to}0$), the terminal event $\{x_0 \in A\}$ lies at the process's endpoint; the $h$-transform applied to the reverse-time PF-ODE carries a negative sign, because the guidance must deflect the trajectory against its natural flow direction to reach the target at $t{=}0$.

\section{Method}
\label{sec:method}

\noindent\textbf{Problem formulation: editing as conditional generation.}
Given the source latent $x_0^{src} = \mathcal{E}(z^{src})$ and a target prompt $c^{edit}$, the goal is to generate an edited latent $x_0^{edit}$ that satisfies both constraints simultaneously. We formulate this as sampling from:
\begin{equation}
x_0^{edit} \sim p(\,\cdot \mid x_0^{src},\, c^{edit}),
\label{eq:editing_objective}
\end{equation}
where the conditional distribution encodes proximity to $x_0^{src}$ in unedited regions and alignment with $c^{edit}$ in edited regions. In the RF framework, this amounts to modifying the backward sampling ODE $dx_t/dt = v_\theta(x_t,t)$ so that the
output remains close to $x_0^{src}$ except where $c^{edit}$ specifies changes. The two objectives are inherently in tension: the velocity direction needed to reconstruct $x_0^{src}$ and the direction needed to align with $c^{edit}$ are generally not parallel and may conflict. How to decouple and effectively balance them is the key of this problem.

\begin{figure}[htbp]
    \centering
    \includegraphics[width=\linewidth]{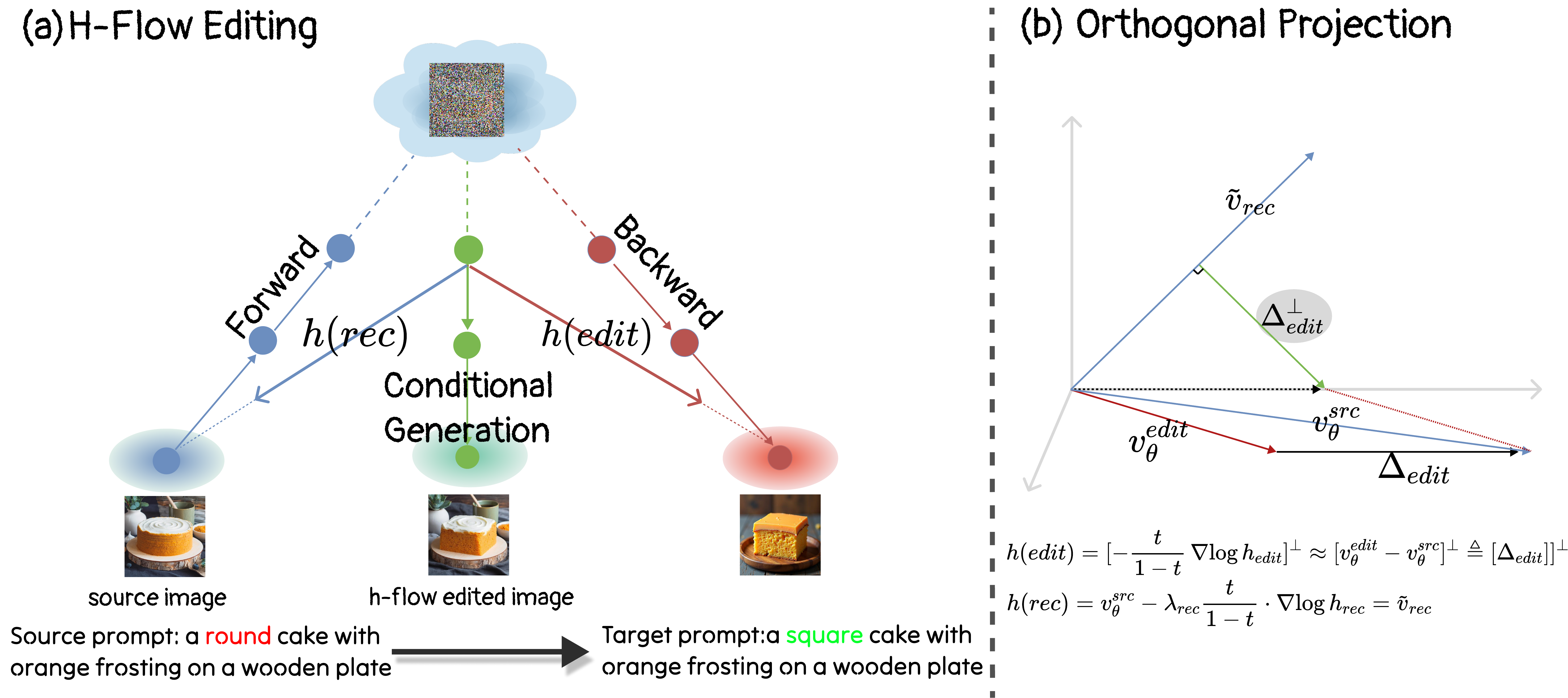}
    \caption{\textbf{Overview of $h$-Flow.} \textbf{(a) Dual-Guidance via $h$-Transforms.} We formulate image editing as a conditional generation process governed by two terminal constraints. The generative trajectory is jointly guided by a reconstruction function $h_{rec}$ (ensuring source structural fidelity) and an editing function $h_{edit}$ (enforcing target semantic alignment). \textbf{(b) Orthogonal Velocity Projection.} To resolve the inherent tension between reconstruction and editing, we project the editing direction $\Delta_{edit}$ onto the orthogonal complement of the reconstruction velocity $\tilde{v}_{rec}$. The isolated orthogonal component $\Delta_{edit}^{\perp}$ exclusively alters the semantic content without interfering with the source structure, enabling perfectly decoupled control.}
    \label{fig:main_pipeline}
\end{figure}

%=============================================
\subsection{Rectified Flow Bridge Construction}
\label{sec:bridge}

A rigorous mathematical tool for such conditioning is Doob's $h$-transform~\cite{doob1984classical}. For a reverse-time
generative SDE $dx_t = f\,dt + g\,d\bar{w}_t$ with marginals $\{p_t\}$, conditioning on a terminal event $\{x_0 \in A\}$
modifies the probability-flow ODE to:
\begin{equation}
\frac{dx_t}{dt} = \widetilde{f}(x_t,t)
- \frac{g^2}{2}\,\nabla_{x_t}\!\log h(x_t,t),
\label{eq:bridge_pfode}
\end{equation}
where $\widetilde{f} = f - \frac{g^2}{2}\nabla\!\log p_t$ is the unconditioned PF-ODE drift, and $h(x_t,t)$ is a product
of harmonic functions encoding the constraints~\cite{du2023reduce} (detailed in the Supplementary Material).
% (see Supplementary~\ref{sup:reverse_time_sign} for the sign derivation)
\noindent\textbf{The deterministic-ODE obstacle.} 
Rectified Flow (RF) generates samples via a purely deterministic ODE where the diffusion coefficient $g(t) \equiv 0$. Consequently, the $h$-transform guidance term $\frac{g^2}{2}\nabla\!\log h$ vanishes identically. This structural incompatibility means the continuous-time $h$-transform cannot be directly applied to RF models. 

\noindent\textbf{Equivalent SDE construction.} 
To circumvent this, we exploit the property that infinitely many SDEs can share the exact same marginal distributions $\{p_t\}$ as a given ODE---they differ only in stochastic fluctuations that cancel out marginally. We construct an equivalent It\^{o} SDE with diffusion coefficient $g_{eq}(t) = \sqrt{2t/(1{-}t)}$, chosen so that $g_{eq}(0){=}0$ (no noise at the clean-data endpoint) and $g_{eq}(t){>}0$ for $t{>}0$. We detail in Supplementary how this construction yields an equivalent process that perfectly matches the RF marginals. 
% Supplementary~\ref{sup:equiv_sde_proof}
Substituting $g_{eq}$ into the bridge PF-ODE~\eqref{eq:bridge_pfode} and noting that $g_{eq}^2/2 = t/(1{-}t)$ and $\widetilde{f}_{eq} = v_\theta$ (see the Supplementary Material), we obtain
%(Supplementary~\ref{sup:equiv_sde_proof})
\begin{equation}
\frac{dx_t}{dt} = v_\theta(x_t,t) - \frac{t}{1-t}\nabla_{x_t}\!\log h(x_t,t),
\label{eq:rf_bridge}
\end{equation}
Recall that the harmonic function $h(x_t, t)$ encodes the probability of reaching a specified terminal event at $t{=}0$. In our image editing formulation, the desired terminal state is jointly governed by two simultaneous conditions: maintaining structural fidelity to the source image ($src$) and achieving semantic alignment with the target prompt ($edit$). By treating these two terminal constraints as conditionally independent events given the intermediate state $x_t$, the overall harmonic function naturally factorizes into a product of two experts:
\begin{equation}
h = h_{rec} \cdot h_{edit} \quad \implies \quad
\nabla\!\log h = \nabla\!\log h_{rec} + \nabla\!\log h_{edit}.
\label{eq:product_h}
\end{equation}
In the log-gradient space, this factorization elegantly translates into an additive superposition of a reconstruction force and an editing force. Because our equivalent SDE strictly shares the RF marginals, this dual-guidance mechanism transfers exactly to the deterministic RF framework without altering the pre-trained model's unconditional prior. This formulation provides a rigorous mathematical bridge from abstract terminal goals to concrete ODE modifications, leading directly to the specific design of each guidance term.
\subsection{$h$-Function Design}
\label{sec:Guidance}

The RF bridge equation~\eqref{eq:rf_bridge} requires specifying the harmonic function $h = h_{rec} \cdot h_{edit}$. We now derive closed-form expressions for each term, converting the abstract bridge equation into a computable editing ODE.

\subsubsection{Source preservation ($h_{rec}$)}

The source constraint corresponds to the terminal event $\{x_0 = x_0^{src}\}$. Setting $h_{rec}(x_0,0) = \delta(x_0 - x_0^{src})$ and applying the harmonic property~\cite{doob1984classical} gives $h_{rec}(x_t,t) = p(x_0^{src} \mid x_t)$. Expanding via Bayes' rule and evaluating the gradients under RF's Gaussian kernel (see the Supplementary Material), we obtain:
%Supplementary~\ref{sup:edit_derivation}
\begin{equation}
\frac{t}{1-t}\cdot\nabla\!\log h_{rec}
= \frac{x_0^{src} - \hat{x}_0}{t}
= v_\theta^{src} - v_{rec},
\label{eq:h_rec_velocity}
\end{equation}
where $\hat{x}_0 = x_t - t\,v_\theta^{src} \approx \mathbb{E}[x_0 \mid x_t, c^{src}]$ is the Tweedie posterior mean
and we define the \textbf{reconstruction velocity}:
\begin{equation}
v_{rec} \;\triangleq\; \frac{x_t - x_0^{src}}{t}.
\label{eq:v_rec}
\end{equation}

This is the unique velocity that steers $x_t$ straight ahead, arriving exactly at $x_0^{src}$ at the end of the generative process ($t{=}0$). Substituting Eq.~\eqref{eq:h_rec_velocity} into the RF bridge equation~\eqref{eq:rf_bridge}, the guidance term $-\frac{t}{1-t}\nabla\!\log h_{rec}$ becomes $v_{rec} - v_\theta^{src}$, and the conditioned velocity reduces
to exactly $v_{rec}$. To allow tunable reconstruction strength, we relax the exact bridge and define:

\begin{equation}
\tilde{v}_{rec} \;=\; v_\theta^{src} \;+\; \lambda_{rec}\,v_{rec},
\label{eq:vbase}
\end{equation}
where $\lambda_{rec} \geq 0$. We adopt this additive form rather than the residual alternative
$v_\theta^{src} + \lambda(v_{rec} - v_\theta^{src})$, because the latter contaminates the error-free $v_{rec}$ with model prediction error through the subtraction $v_{rec} - v_\theta^{src}$; the additive form preserves $v_{rec}$ as a clean structural anchor (see Supplementary for further discussion on this parameterization).

\subsubsection{Text-guided editing ($h_{edit}$)}

The editing constraint encodes the target prompt $c^{edit}$. We set $h_{edit}(x_0,0) = p(c^{edit} \mid x_0)$, which propagates to $h_{edit}(x_t,t) = p(c^{edit} \mid x_t)$ by the harmonic property. Unlike $h_{rec}$, this score has no closed form.
Expanding $\nabla\!\log h_{edit}$ via Bayes' rule yields a
difference of conditional scores. To isolate the semantic change
from $c^{src}$ to $c^{edit}$, we replace the unconditional
baseline $\nabla\!\log p_t(x_t)$ with the source-conditioned
score $\nabla\!\log p(x_t \mid c^{src})$, effectively computing
$\nabla\!\log[p(c^{edit} \mid x_t) / p(c^{src} \mid x_t)]$.
Approximating both conditional scores via their Tweedie estimates
and multiplying by the bridge coefficient $t/(1{-}t)$, all
time-dependent factors cancel exactly
(As shown in the Supplementary Material), yielding:
%(Supplementary~\ref{sup:edit_derivation})
\begin{equation}
-\frac{t}{1-t}\,\nabla\!\log h_{edit}
\;\approx\; v_\theta^{edit} - v_\theta^{src}
\;\triangleq\; \Delta_{edit}.
\label{eq:delta_edit}
\end{equation}
We call $\Delta_{edit}$ the \textbf{editing direction}. Compared
to the standard CFG direction
$v_\theta^{edit} - v_\theta^{\varnothing}$, our formulation
discards the shared content between $c^{edit}$ and $c^{src}$ and
retains only the novel semantic change.

% ============================================================
% §4.4 Orthogonal Decomposition
% ============================================================

\subsection{Orthogonal Decomposition}
\label{sec:ortho}

The product-of-experts framework combines the two guidance terms additively. However, in high-dimensional velocity space, $\Delta_{edit} = v_\theta^{edit} - v_\theta^{src}$ may retain a component along the reconstruction-guided direction $\tilde{v}_{rec}$ that interferes with source preservation. We project $\Delta_{edit}$ onto the orthogonal complement of $\tilde{v}_{rec}$:
\begin{equation}
\Delta_{edit}^{\perp} = \Delta_{edit} - \frac{\langle \Delta_{edit},\, \tilde{v}_{rec}\rangle}{\|\tilde{v}_{rec}\|^2}\,\tilde{v}_{rec}.
\label{eq:ortho_proj}
\end{equation}
By construction, $\langle \Delta_{edit}^{\perp},\, \tilde{v}_{rec}\rangle = 0$: the editing signal now operates exclusively in directions orthogonal to the reconstruction flow. This is analogous to projected gradient descent in constrained optimization, where updates are projected onto the feasible subspace.

The projection decomposes $\Delta_{edit}$ into two interpretable components:
\begin{equation}
\Delta_{edit} = \underbrace{\frac{\langle \Delta_{edit},\, \tilde{v}_{rec}\rangle}{\|\tilde{v}_{rec}\|^2}\,\tilde{v}_{rec}}_{\text{parallel to } \tilde{v}_{rec}\;(\text{discarded})} \;+\; \underbrace{\Delta_{edit}^{\perp}}_{\text{orthogonal to } \tilde{v}_{rec}\;(\text{kept})}.
\label{eq:decomp}
\end{equation}
The parallel component overlaps with the reconstruction direction and would either redundantly reinforce or destructively counteract it; discarding it eliminates the fidelity--editability entanglement. The orthogonal component carries only the \emph{novel semantic content} of $c^{edit}$ that is absent from the source-preserving trajectory.

Varying $\lambda_{edit}$ now modifies the edit without degrading fidelity, and varying $\lambda_{rec}$ adjusts fidelity without suppressing the edit direction. The two controls are fully decoupled. The computational cost is one inner product and one vector subtraction per step.

% ============================================================
% §4.5 Bridge Traversal
% ============================================================
\subsection{The Complete $h$-Flow Framework}
\label{sec:pipeline}

Starting from the RF bridge equation~\eqref{eq:rf_bridge} and incorporating the reconstruction guidance~\eqref{eq:vbase}, editing direction~\eqref{eq:delta_edit}, and orthogonal projection~\eqref{eq:ortho_proj}, we arrive at the editing ODE:
\begin{equation}
\frac{dx_t}{dt}
= \underbrace{v_\theta^{src} + \lambda_{rec}\,v_{rec}}_{\tilde{v}_{rec}}
+ \;\lambda_{edit}\,\Delta_{edit}^{\perp}.
\label{eq:edit_ode}
\end{equation}
By construction, $\tilde{v}_{rec}$ and $\Delta_{edit}^{\perp}$ are orthogonal: the two controls $\lambda_{rec}$ and $\lambda_{edit}$ act on independent subspaces without mutual interference.

The complete procedure is given in Algorithm~\ref{alg:editing}. The method is training-free and
gradient-free; its only interface to the pre-trained model is two forward passes of $v_\theta$ per step. The backward editing mechanism is agnostic to the forward initialization---any inversion, noise injection, or hybrid strategy is compatible, and improvements in the forward process translate directly to better results.

\begin{algorithm}[t]
\caption{$h$-Flow: RF Bridge Editing}
\label{alg:editing}
\begin{algorithmic}[1]
\REQUIRE Source $z^{src}$, prompts $c^{src}$ and $c^{edit}$, $v_\theta$, VAE $\mathcal{E}/\mathcal{D}$, $\lambda_{rec}$, $\lambda_{edit}$, schedule $\{t_N > \cdots > t_0{=}0\}$
\STATE $x_0^{src} \leftarrow \mathcal{E}(z^{src})$
\STATE $x_{t_N} \leftarrow \textsc{Invert}(x_0^{src}, v_\theta, c^{src})$
    % \hfill $\triangleright$ Modular forward process
\FOR{$n = N, \ldots, 1$}
    \STATE $t \leftarrow t_n$; \quad $\delta t \leftarrow t_n - t_{n-1}$
    % \STATE \textit{// Reconstruction guidance}
    \STATE $v^{src} \leftarrow v_\theta(x_t, t, c^{src})$
    \STATE $v_{rec} \leftarrow (x_t - x_0^{src})\,/\,t$
        % \hfill $\triangleright$ Eq.~(\ref{eq:v_rec}): exact reconstruction velocity
    \STATE $\tilde{v}_{rec} \leftarrow v^{src} + \lambda_{rec} \cdot v_{rec}$
        % \hfill $\triangleright$ Eq.~(\ref{eq:vbase})
    % \STATE \textit{// Editing: orthogonal to $\tilde{v}_{rec}$}
    \STATE $v^{edit} \leftarrow v_\theta(x_t, t, c^{edit})$
    \STATE $\Delta_{edit} \leftarrow v^{edit} - v^{src}$
        % \hfill $\triangleright$ Eq.~(\ref{eq:delta_edit})
    \STATE $\Delta_{edit}^{\perp} \leftarrow \Delta_{edit} - \dfrac{\langle \Delta_{edit},\,\tilde{v}_{rec}\rangle}{\|\tilde{v}_{rec}\|^2}\,\tilde{v}_{rec}$
        % \hfill $\triangleright$ Eq.~(\ref{eq:ortho_proj})
    % \STATE \textit{// Bridge step}
    \STATE $v \leftarrow \tilde{v}_{rec} + \lambda_{edit} \cdot \Delta_{edit}^{\perp}$
    \STATE $x_{t_{n-1}} \leftarrow x_t - v \cdot \delta t$
\ENDFOR
\RETURN $\mathcal{D}(x_{t_0})$
\end{algorithmic}
\end{algorithm}

\section{Experiments}
\label{sec:experiments}

\begin{figure}[h!] % <--- 改为 [h!]
  \centering
  \includegraphics[width=1\linewidth]{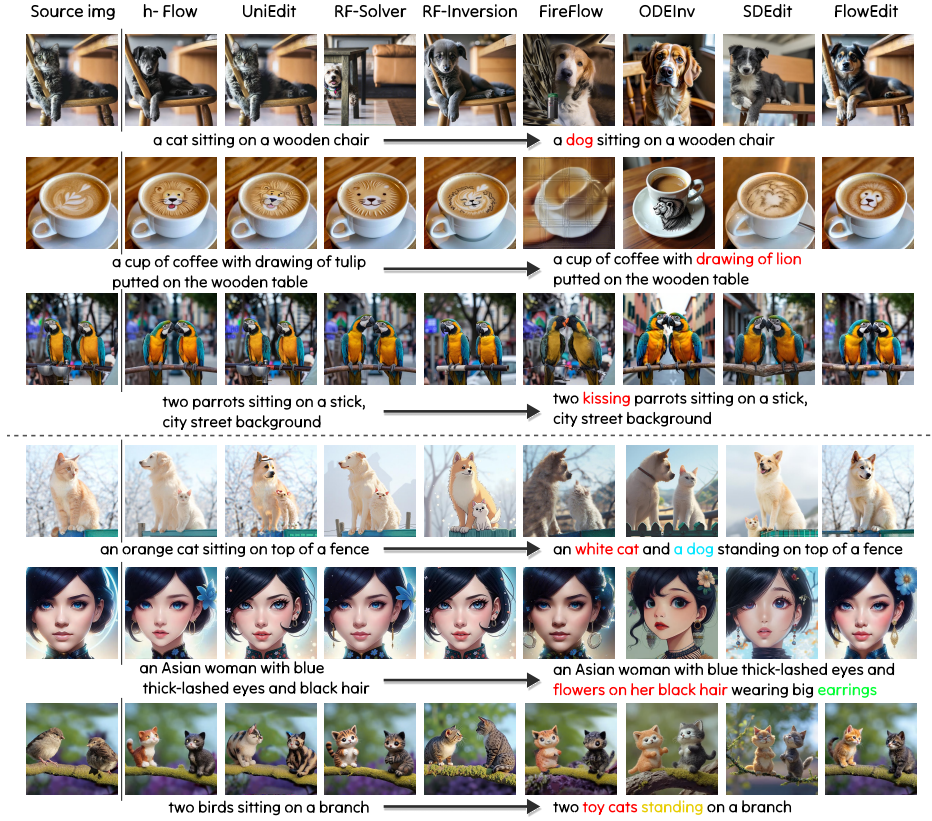}
  \caption{\textbf{Qualitative comparisons.} The source prompt and the target prompt is shown on the bottom for each example. The first column displays the real input image. The subsequent columns present the editing results of our proposed $h$-Flow, followed by various state-of-the-art baselines. }
  \vspace{-2em}
  \label{fig:pie-compare}
\end{figure}

\subsection{Experimental Setup}
\label{sec:setup}

% TODO: 填入实际超参数
\noindent{\textbf{{Dataset.}}
We evaluate on the PIE-Bench~\cite{ju2024pnp} and PIE-Bench++~\cite{huang2024paralleledits} datasets, which contain 700 diverse images covering humans, animals, and objects in varied environments. Each sample includes source and target text descriptions and an annotated mask indicating the editing region. PIE-Bench focuses on single-aspect edits (e.g., object addition/removal, replacement, color, material, style, background, action, texture, and position), while PIE-Bench++ extends the same images with multi-aspect prompts requiring simultaneous modifications to multiple objects or attributes, enabling evaluation under more complex editing scenarios.

\noindent{\textbf{{Metrics.}} Following~\cite{ju2024pnp}, we report two families of metrics. \emph{Source Consistency} is measured on non-edited regions (defined by the provided masks) via PSNR, LPIPS~\cite{zhang2018unreasonable}, SSIM~\cite{wang2004image}, and MSE, and on the full image via structure distance~\cite{tumanyan2022splicing}. \emph{Target Alignment} is measured by CLIP similarity~\cite{radford2021learning} between the edited image and the target prompt, both on the entire image (CLIP-Entire) and on the edited region (CLIP-Edited). Besides above standard metrics, we further evaluated \emph{Image Quality} with {HPSv2}~\cite{wu2023human} and {Aesthetic Score (AS)}~\cite{discus0434_aesthetic_v25}, which 
quantify perceptual fidelity and visual quality.

\noindent{\textbf{{Baselines.}}
We compare $h$-Flow against seven baselines across two categories.
\emph{Inversion-based}:
RF-Inversion~\cite{rout2025semantic},
RF-Solver~\cite{wang2025taming},
UniEdit-Flow~\cite{jiao2025unieditflowunleashinginversionediting},
FireFlow~\cite{dengfireflow},
and standard ODE inversion (OdeInv).
\emph{Inversion-free}:
FlowEdit~\cite{kulikov2025flowedit}.
SDEdit~\cite{mengsdedit}.
All methods use officially released code with recommended
hyperparameters and are evaluated on the same FLUX.1-dev backbone. 
% \zhen{too long}
% 已修改

\noindent{\textbf{Implementation Details}}

We evaluate on FLUX.1-dev. For the backward editing process, we use $N{=}28$ Euler steps with a schedule on $[0,1]$. Unless otherwise specified, the forward process defaults to RF-Inversion~\cite{rout2025semantic} with $n{=}24$ steps. The hyperparameters $\lambda_{rec}$ and $\lambda_{edit}$ are set to \texttt{1.2} and \texttt{1.5}, respectively.

\subsection{Experimental Results}
\label{Experimental Results}

\begin{table*}[t]
\centering
\caption{Quantitative comparison on PIE-Bench and PIE-Bench++. Best in \textbf{bold}, second best \underline{underlined}, third best \uwave{wavy}. ``Rank'' indicates average rank across metrics.}
\label{tab:pie}
\resizebox{\textwidth}{!}{
% 定义11列：2列基础 + 5列Source + 2列Target + 2列Image Quality
\begin{tabular}{lc|ccccc|cc|cc}
\toprule
\multirow{2}{*}{Method} & \multirow{2}{*}{Rank$\downarrow$} & \multicolumn{5}{c|}{\emph{Source Consistency}} & \multicolumn{2}{c|}{\emph{Target Alignment}} & \multicolumn{2}{c}{\emph{Image Quality}} \\
\cmidrule(lr){3-7} \cmidrule(lr){8-9} \cmidrule(lr){10-11}
& & PSNR$\uparrow$ & LPIPS$\downarrow$ & SSIM$\uparrow$ & MSE$\downarrow$ & Distance$\downarrow$ & CLIP-Entire$\uparrow$ & CLIP-Edited$\uparrow$ & HPS$\uparrow$ & AS$\uparrow$ \\
\midrule

\multicolumn{11}{c}{\textbf{PIE-Bench}}\\
\midrule
\multicolumn{11}{l}{\emph{Inversion-Free}}\\

SDEdit~\cite{mengsdedit}
& 6.06
& 18.01
& 0.240
& 0.650
& 0.021
& 0.063
& 24.67
& 22.54 
& 27.08 & \underline{5.05} \\

FlowEdit~\cite{kulikov2025flowedit}
& 4.39
& 21.92
& \uwave{0.111}
& \uwave{0.833}
& 0.009
& \uwave{0.028}
& 25.19
& 22.54 
& 27.70 & 4.87 \\

\midrule
\multicolumn{11}{l}{\emph{Inversion-Based}}\\

OdeInv
& 5.44
& 13.44
& 0.349
& 0.620
& 0.071
& 0.130
& \textbf{26.56}
& \textbf{23.83} 
& \textbf{28.73} & 4.92 \\

FireFlow~\cite{dengfireflow}
& 5.89
& 18.15
& 0.207
& 0.738
& 0.026
& 0.057
& 25.69
& 22.56 
& 26.45 & 4.63 \\

Rf-Inversion~\cite{rout2025semantic}
& 5.00
& 20.59
& 0.186
& 0.706
& 0.013
& 0.042
& 25.26
& 22.34 
& \underline{28.21} & \uwave{4.98} \\

UniEdit-Flow~\cite{jiao2025unieditflowunleashinginversionediting}
& 3.00
& \textbf{29.45}
& \textbf{0.060}
& \textbf{0.900}
& \textbf{0.002}
& \textbf{0.011}
& \uwave{25.80}
& 22.34 
& 26.77 & 4.96 \\

RF-Solver~\cite{wang2025taming}
& 3.22
& \uwave{22.90}
& 0.140
& 0.819
& \uwave{0.007}
& 0.031
& \underline{26.09}
& \uwave{22.68} 
& 27.24 & \textbf{5.48} \\

\midrule
$h$-Flow (ours)
& \textbf{2.33}
& \underline{23.88}
& \underline{0.110}
& \underline{0.840}
& \underline{0.006}
& \underline{0.026}
& \uwave{25.80}
& \underline{22.80} 
& \uwave{27.79} & \uwave{4.98} \\
\midrule
\multicolumn{11}{c}{\textbf{PIE-Bench++}}\\
\midrule
\multicolumn{11}{l}{\emph{Inversion-Free}}\\

SDEdit~\cite{mengsdedit}
& 7.11
& 19.92
& 0.2300
& 0.680
& 0.014
& 0.090
& 22.50
& 21.80 
& 27.08 & 4.00 \\

FlowEdit~\cite{kulikov2025flowedit}
& 4.11
& 22.27
& \underline{0.1087}
& \uwave{0.841}
& \uwave{0.009}
& \uwave{0.031}
& 24.94
& 24.24 
& \uwave{27.70} & 4.86 \\

\midrule
\multicolumn{11}{l}{\emph{Inversion-Based}}\\

OdeInv
& 5.44
& 17.21
& 0.2400
& 0.720
& 0.028
& 0.080
& \uwave{25.50}
& \textbf{24.92} 
& \textbf{28.20} & 4.73 \\

FireFlow~\cite{dengfireflow}
& 5.67
& 18.95
& 0.1877
& 0.760
& 0.021
& 0.060
& \underline{25.71}
& \uwave{24.77} 
& 26.45 & 4.56 \\

Rf-Inversion~\cite{rout2025semantic}
& 4.44
& 21.13
& 0.1800
& 0.730
& 0.011
& 0.040
& 25.23
& 24.49 
& \underline{27.93} & \underline{5.18} \\

UniEdit-Flow~\cite{jiao2025unieditflowunleashinginversionediting}
& 3.56
& \textbf{29.46}
& \textbf{0.0600}
& \textbf{0.900}
& 0.018
& \textbf{0.018}
& 25.37
& 24.39 
& 26.96 & 4.89 \\

RF-Solver~\cite{wang2025taming}
& 2.89
& \uwave{23.69}
& 0.1285
& 0.832
& \underline{0.006}
& 0.033
& \textbf{25.79}
& \underline{24.89} 
& 27.24 & \textbf{5.43} \\

\midrule
$h$-Flow (ours)
& \textbf{2.78}
& \underline{24.30}
& \uwave{0.1090}
& \underline{0.850}
& \textbf{0.005}
& \underline{0.029}
& 25.48
& 24.54 
& 27.67 & \uwave{4.92} \\

\bottomrule
\end{tabular}
}
\end{table*}

% \noindent{\textbf{Quantitative evaluation}} From Table~\ref{tab:pie}, we have the following observations 1) For both benchmark, xxx achieves high 2) xxx achieves xxx 3) our achieves better trade-off 

\noindent{\textbf{Qualitative evaluation.}}
As shown in Figure~\ref{fig:pie-compare}, we can observe 1) \emph{Superior semantic alignment.} $h$-Flow accurately executes complex prompts, including subject translation (\eg, cat to dog in row 1, col 2) and localized element additions (\eg, adding flowers and earrings in row 5, col 2), avoiding the loss of subject identity commonly seen in ODEInv (col 7) and SDEdit (col 8).  2) \emph{Strict structural preservation.} Compared to FireFlow (col 6), our method perfectly preserves the background and unedited regions.  3) \emph{Optimal editing-fidelity balance.} Both visually and quantitatively, $h$-Flow demonstrates outstanding performance in balancing editability and fidelity. For instance, UniEdit (row 3, col 3) fails to achieve the editing action of ``kissing'' for the parrots, whereas our method (row 3, col 2) successfully implements this action while maintaining high structural consistency, achieving top-tier aesthetic quality and text alignment.

\noindent{\textbf{Quantitative evaluation}} From Table~\ref{tab:pie}, we have the following observations 1) For both benchmarks, UniEdit achieves the highest source consistency (\eg, highest PSNR and lowest MSE) but yields suboptimal target alignment and lower human preference scores (HPS). 2) Conversely, ODEInv achieves high target alignment (\eg, best CLIP-Edited scores) but severely degrades image quality and structural preservation, as evidenced by its lowest PSNR and highest Distance metrics. 3) Our proposed $h$-Flow achieves a significantly better trade-off. It ranks first in average across all 9 metrics on both PIE-Bench and PIE-Bench++, maintaining highly competitive source consistency while delivering top-tier target alignment and aesthetic quality (AS).

\begin{figure*}[t]
    \centering
    
    % --- 左半边：缩小图片展示空间，改为 0.40 (40% 宽度) ---
    \begin{minipage}[c]{0.40\linewidth}
        \centering
        \includegraphics[width=\linewidth]{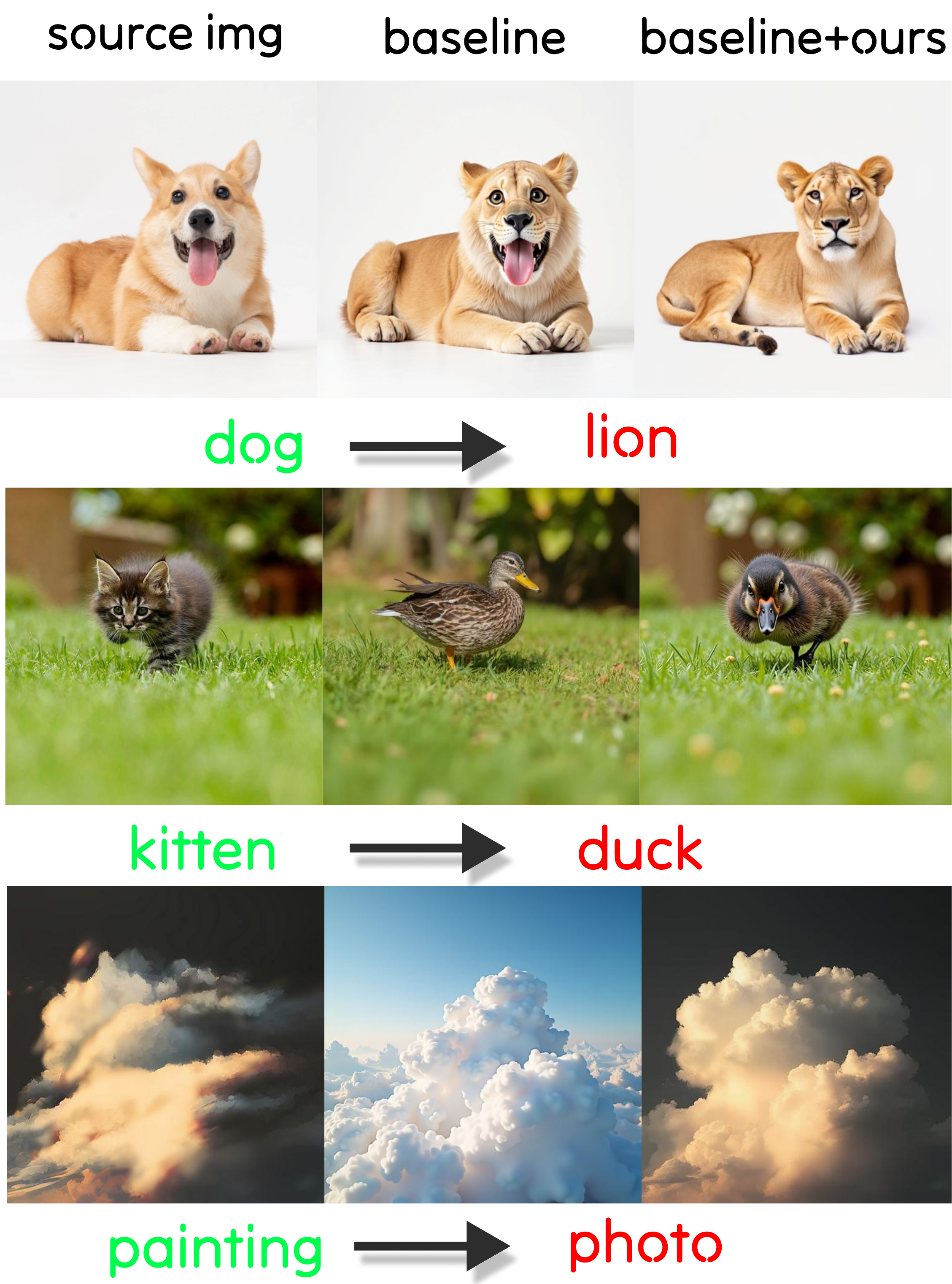} 
    \end{minipage}
    \hfill
    % --- 右半边：增大表格和文字说明的空间，改为 0.56 (56% 宽度) ---
    \begin{minipage}[c]{0.56\linewidth}
        \centering
        
        % 1. 右侧顶部：由于外部 minipage 变宽，resizebox 会自动将表格放大拉宽
        \resizebox{\linewidth}{!}{
        \begin{tabular}{l|cc|cc}
        \toprule
        \multirow{2}{*}{Method} & \multicolumn{2}{c|}{\emph{Source Consistency}} & \multicolumn{2}{c}{\emph{Target Alignment}} \\
        \cmidrule(lr){2-3} \cmidrule(lr){4-5}
        & SSIM$\uparrow$ & Distance$\downarrow$ & CLIP-Entire$\uparrow$ & CLIP-Edited$\uparrow$ \\
        \midrule
        RF-Inversion~\cite{rout2025semantic} 
        & 0.70 & \textbf{0.070} 
        & 25.26 & 22.34 \\
        \quad + $h$-Flow (ours)
        & \textbf{0.73} & 0.082 
        & \textbf{26.10} & \textbf{23.60} \\
        \midrule
        SDEdit~\cite{mengsdedit}
        & 0.65 & 0.063 
        & 24.67 & 22.54 \\
        \quad + $h$-Flow (ours)
        & \textbf{0.65} & \textbf{0.063} 
        & \textbf{25.28} & \textbf{23.10} \\
        \midrule
        ODE-Inv 
        & 0.62 & 0.130 
        & \textbf{26.56} & \textbf{23.83} \\
        \quad + $h$-Flow (ours)
        & \textbf{0.70} & \textbf{0.077} 
        & 26.44 & 23.65 \\
        \bottomrule
        \end{tabular}
        }
        
        \vspace{8pt} % 表格和文字之间留出呼吸空间
        
        % 2. 右侧中部：表格的标题说明
        \captionof{table}{$\boldsymbol{\uparrow}$ Quantitative comparison of the compatibility of $h$-Flow with diverse forward processes on PIE-Bench.}
        \label{tab:forward_ablation}
        
        \vspace{8pt} % 两个标题之间留出较大间距，形成视觉分隔
        
        % 3. 右侧底部：图片的标题说明
        \captionof{figure}{$\boldsymbol{\leftarrow}$ Qualitative comparison of $h$-Flow integrated with different forward processes (baselines from top to bottom: RF-Inversion, SDEdit, and ODE-Inv).}
        \label{fig:pie-compare-1}
        
    \end{minipage}
    
\end{figure*}
\subsection{Ablation Study}
\label{Ablation Study}
\noindent{\textbf{Compatibility with Forward Processes.}} Since our backward editing ODE is agnostic to the forward process, $h$-Flow can be integrated with diverse forward initialization strategies in a plug-and-play manner. To validate this, we combine $h$-Flow with three distinct forward processes: RF-Inversion, SDEdit (noise injection), and ODE-Inv. As shown in Table~\ref{tab:forward_ablation}, applying $h$-Flow consistently optimizes the trade-off between editability and fidelity across all settings. For RF-Inversion and SDEdit, $h$-Flow substantially improves target text alignment (CLIP-Edited gains of +1.26 and +0.56, respectively) while strictly preserving structural metrics. Notably, for ODE-Inv, the base forward process achieves high editability but suffers from severe structural degradation (SSIM drops to 0.62, Distance spikes to 0.130). In this case, integrating $h$-Flow acts as a powerful structural stabilizer, massively recovering source fidelity (SSIM surges to 0.70, Distance drops to 0.077) with only a negligible fluctuation in text alignment. As shown in Figure~\ref{fig:pie-compare-1}, $h$-Flow consistently delivers high-quality edits, robustly compensating for the inherent weaknesses of each base inversion method.

\noindent{\textbf{Hyperparameters.}} {1) Effect of $\lambda_{rec}$.} Table~\ref{tab:ablation_rec} sweeps $\lambda_{rec}$ with
$\lambda_{edit}{=}1.5$ fixed. As $\lambda_{rec}$ increases from 0 to 1.5, confirming that the reconstruction guidance progressively anchors the trajectory to the source. {We set 1.2 for trade-off.} {2) Effect of $\lambda_{edit}$.} Table~\ref{tab:ablation_edit} sweeps $\lambda_{edit}$ with $\lambda_{rec}{=}1.2$ fixed. Increasing $\lambda_{edit}$ from 0 to 1.5 significantly boosts editability (CLIP-Edited $20.49 \to 22.80$) with only a modest fidelity drop. Further increasing to 3.0 yields diminishing returns in editing but exacerbates fidelity degradation, thus we set $\lambda_{edit}{=}1.5$ as a suitable trade-off. Crucially, comparing Tables~\ref{tab:ablation_rec} and~\ref{tab:ablation_edit} demonstrates successful objective decoupling via our decomposition: sweeping $\lambda_{rec}$ barely affects editability (a mere 0.19-point drop in CLIP-Edited), whereas sweeping $\lambda_{edit}$ drives a massive 3.0-point swing, where each parameter predominantly controls its intended axis.

\begin{table*}[t]
    \centering
    
    % --- 第一排：放置两个 Caption（底部对齐 [b]，保证下方表格的起点绝对一致） ---
    \begin{minipage}[b]{0.48\linewidth}
        \caption{Effect of reconstruction strength $\lambda_{rec}$ (fixed $\lambda_{edit}{=}1.5$).}
        \label{tab:ablation_rec}
    \end{minipage}
    \hfill
    \begin{minipage}[b]{0.48\linewidth}
        \caption{Effect of editing strength $\lambda_{edit}$ (fixed $\lambda_{rec}{=}1.2$).}
        \label{tab:ablation_edit}
    \end{minipage}
    
    % \vspace{4pt} % 标题和表格之间保留优雅的呼吸间距
    
    % --- 第二排：放置两个 表格主体（顶部对齐 [t]，绝对水平对称） ---
    \begin{minipage}[t]{0.48\linewidth}
        \centering
        \resizebox{\linewidth}{!}{
        \begin{tabular}{c|cccc}
        \toprule
        $\lambda_{rec}$ & SSIM$\uparrow$ & Distance$\downarrow$ & CLIP-Entire$\uparrow$ & CLIP-Edited$\uparrow$ \\
        \midrule
        0.0 & 0.830 & 0.026 & \textbf{25.97} & \textbf{22.92} \\
        1.2 & 0.840 & 0.026 & 25.80 & 22.80 \\
        1.5 & \textbf{0.856} & \textbf{0.025} & 25.65 & 22.73 \\
        \bottomrule
        \end{tabular}
        }
    \end{minipage}
    \hfill
    \begin{minipage}[t]{0.48\linewidth}
        \centering
        \resizebox{\linewidth}{!}{
        \begin{tabular}{c|cccc}
        \toprule
        $\lambda_{edit}$ & SSIM$\uparrow$ & Distance$\downarrow$ & CLIP-Entire$\uparrow$ & CLIP-Edited$\uparrow$ \\
        \midrule
        0.0 & \textbf{0.846} & \textbf{0.020} & 22.84 & 20.49 \\
        1.5 & 0.840 & 0.026 & 25.80 & 22.80 \\
        3.0 & 0.804 & 0.034 & \textbf{26.26} & \textbf{23.49} \\
        \bottomrule
        \end{tabular}
        }
    \end{minipage}
    
\end{table*}

% --- 使用 wraptable 实现文字环绕 ---
% {r} 代表靠右侧 (right)，{0.48\linewidth} 代表表格占据 48% 的列宽
\begin{wraptable}{r}{0.5\linewidth}
    \vspace{-3em} % 微调：抵消 LaTeX 默认的顶部额外留白，让表格和第一行字绝对顶部对齐
    \centering
    \caption{Effect of orthogonal projection. The projection improves editing effectiveness without sacrificing source consistency.}
    \label{tab:ablation_proj}
    \resizebox{\linewidth}{!}{
    \begin{tabular}{l|cccc}
    \toprule
    Variant & SSIM$\uparrow$ & Distance$\downarrow$ & CLIP-Entire$\uparrow$ & CLIP-Edited$\uparrow$ \\
    \midrule
    w/o projection
    & 0.834 & \textbf{0.026}
    & 25.37 & 22.79 \\
    w/ projection
    & \textbf{0.840} & \textbf{0.026}
    & \textbf{25.80} & \textbf{22.80} \\
    \bottomrule
    \end{tabular}
    }
    \vspace{-1em} % 微调：缩减表格下方的默认留白，让文字更紧凑地包裹在表格下方
\end{wraptable}

% --- 文字部分紧跟其后，不需要加任何换行或额外排版，它会自动流淌并环绕 ---
\noindent\textbf{Effect of orthogonal projection.} Table~\ref{tab:ablation_proj} ablates projecting $\Delta_{edit}$ onto the orthogonal complement of $\tilde{v}_{rec}$. Under fixed hyperparameters ($\lambda_{rec}{=}1.2$, $\lambda_{edit}{=}1.5$), the projection significantly improves CLIP-Entire (+0.43) while keeping structural source fidelity strictly preserved (Distance remains completely unchanged at 0.026, and SSIM slightly improves to 0.840). Mechanistically, the raw $\Delta_{edit}$ contains components along $\tilde{v}_{rec}$ that inadvertently counteract reconstruction guidance. Eliminating this interference focuses the editing strictly on orthogonal semantic shifts, achieving enhanced text alignment at zero structural cost.

\section{Conclusion}
We presented \textbf{h-flow}, a principled flow-based image editing method that reformulates editing through Doob’s $h$-Transform and extends it to the deterministic rectified flow framework. By explicitly modeling source consistency and target alignment with dedicated $h$-functions and decoupling their effects via velocity orthogonal decomposition, our approach achieves a controllable trade-off between reconstruction and semantic editing while remaining training-free and plug-and-play. Future work will explore extending this framework to other modalities such as video or 3D generation and incorporating richer combinations of control conditions for more general multimodal editing.

\section{Acknowledgement}

This work was supported by the National Natural Science Foundation of China (62572194, 62472178, 62376244) and the Shanghai Frontiers Science Center of Molecule Intelligent Syntheses. This article was also supported by the Key Technology Research and Development Program Project of the Shanghai Science and Technology Commission under Grants (25511107200, 25511102400). Long Chen was supported by Hong Kong SAR Research Grants Council (RGC) General Research Fund (16219025), National Natural Science Foundation of China (NSFC) Young Scientists Fund Category B (62522216), Hong Kong SAR Research Grants Council (RGC) Early Career Scheme (26208924), and National Natural Science Foundation of China (NSFC) Young Scientists Fund Category C (62402408).

\bibliographystyle{splncs04}
\bibliography{main}

\newpage

\appendix
\setcounter{equation}{19}
\setcounter{section}{0}
\renewcommand{\thesection}{\Alph{section}}

\section*{Supplementary Material}

Text-based image editing requires a generated sample $x_0$ to satisfy two terminal constraints at $t{=}0$: source consistency to $x_0^{src}$ and target alignment with $c^{edit}$. This is precisely the problem of conditioning a stochastic process on terminal events---the classical setting of Doob's $h$-transform~\cite{doob1984classical}. 

The core challenge of applying this framework to Rectified Flow (RF) is that RF relies on a deterministic ODE, which structurally prohibits the standard $h$-transform. This supplementary material details how we mathematically resolve this obstacle, translating abstract probability constraints into robust geometric algorithms, and provides extended experimental validations. Throughout the supplementary material, we adopt the same mathematical conventions as the main text.

\medskip
\noindent\textbf{Supplementary Organization.}
This supplementary is organized as follows:
\begin{itemize}
  \item Section~\ref{sup:sign} justifies the negative sign
    in the RF bridge equation for reverse-time editing.
  \item Section~\ref{sup:equiv_sde} proves that our
    constructed It\^{o} SDE strictly shares the identical
    marginal distributions with the deterministic RF ODE.
  \item Section~\ref{sup:guidance} derives the closed-form
    reconstruction guidance and editing direction vectors.
  \item Section~\ref{sup:additive} analyses the additive
    guidance parameterization and its robustness against
    model prediction errors.
  \item Section~\ref{sup:ortho} formalizes the orthogonal
    decoupling guarantee and discusses numerical stability.
  \item Section~\ref{sup:complex} extends $h$-Flow to
    complex multi-target editing via prompt decomposition
    and per-sub-edit orthogonal projection.
% \item Section~\ref{sup:ablations} visually validates the
%   orthogonal projection through temporal heatmaps and
%   RGB examples.
\item Section~\ref{sup:additional_comparisons} provides
  additional comparative experiments, including
  generalization to other flow models, evaluation on
  additional benchmarks, and compatibility with stronger
  inversion methods.
\item Section~\ref{sup:extended_qualitative} demonstrates
  extended qualitative results across diverse editing
  scenarios.
\item Section~\ref{sup:limitations} discusses limitations
  and analyses typical failure cases.
\end{itemize}

% ============================================================
\section{Sign of the Guidance Term in the RF Bridge}
\label{sup:sign}
% ============================================================
% Proves: the negative sign in Eq.~\ref{eq:rf_bridge} (main text §4.1)
% Eq.~\ref{eq:rf_bridge
This section justifies the negative sign in the RF bridge equation (Eq.11) in the main text Sec. 4.1.

In the classical formulation, applying Doob's $h$-transform to a forward-time SDE $dx_t = f(x_t, t)\,dt + g(t)\,dw_t$ introduces an additive guidance term $+g^2 \nabla_{x_t}\!\log h$ to the drift, steering the process toward the conditioning event.

However, Rectified Flow generates samples by integrating the ODE backward in time, from $t{=}1$ (noise) to $t{=}0$ (data), meaning the integration step is negative ($dt < 0$). To ensure the trajectory is displaced toward states with a higher probability of reaching the target, the applied vector field must compensate for this reverse temporal direction. Consequently, the correct conditioned PF-ODE must subtract the guidance term:
\begin{equation}
\frac{dx_t}{dt} = v_\theta(x_t,t) - \frac{t}{1-t}\nabla_{x_t}\!\log h(x_t,t).
\end{equation}
This negative sign ensures that a discrete Euler step $\Delta x \approx (dx_t/dt) \Delta t$ produces a positive net displacement along the gradient $\nabla\!\log h$ (since $\Delta t < 0$), effectively driving the generative process toward the desired terminal constraints.

% ============================================================
\section{Equivalent SDE: Proof and Properties}
\label{sup:equiv_sde}
% ============================================================
% Proves: the marginal equivalence claimed in Sec.~4.1 of the main text

This section provides the complete proof that the equivalent SDE constructed in Sec.~4.1 of the main text shares the same marginal distributions as the RF ODE. To legally apply Doob's $h$-transform within the deterministic Rectified Flow (RF) framework, we first construct an equivalent It\^{o} SDE that strictly shares the identical marginal distribution family $\{p_t\}$.

\subsection{Premises and Forward-to-Reverse Connection}
According to the Score-SDE framework~\cite{songscore}, for any forward SDE $dx_t = f(x_t, t)\,dt + g(t)\,dw_t$, there inherently exists a corresponding probability flow ODE (PF-ODE) that shares the same marginals $p_t$, given by the form $dx_t = [f - \frac{g^2}{2}\nabla\!\log p_t]\,dt$.

Crucially, in continuous-time diffusion models, the reverse-time generating ODE is entirely governed by the marginal score $\nabla\!\log p_t$ of the forward process. Direct application of Doob's $h$-transform to a deterministic reverse ODE is mathematically ill-posed, as the absence of noise causes the transition probabilities to degenerate into Dirac delta functions, prohibiting score evaluation. To resolve this, we construct an equivalent \emph{forward} SDE to provide smooth, well-defined transition kernels $p(x_t \mid x_0)$. We then apply the $h$-transform to this forward process to derive the conditional score $\nabla\!\log p(x_t \mid c)$, which is ultimately substituted back into the reverse PF-ODE formulation for conditional sampling.

To make this forward SDE macroscopically equivalent to the deterministic RF ODE $dx/dt = v_\theta$, we force its PF-ODE drift to match the RF velocity field exactly:
\begin{equation}
f_{eq} - \frac{g^2}{2}\nabla\!\log p_t = v_\theta \quad \implies \quad f_{eq} = v_\theta + \frac{g^2}{2}\nabla\!\log p_t.
\end{equation}
Meanwhile, to satisfy the noise-free requirement at the clean data endpoint ($g(0)=0$) and to provide clean algebraic cancellations for the subsequent $h$-transform, we set the diffusion coefficient to $g_{eq}(t) = \sqrt{2t/(1{-}t)}$. The resulting equivalent SDE is:
\begin{equation}
dx_t = \left(v_\theta + \frac{t}{1{-}t}\nabla\!\log p_t\right) dt + \sqrt{\frac{2t}{1{-}t}}\,dw_t.
\label{eq:sup_equiv_sde}
\end{equation}

\subsection{Proof of Marginal Equivalence and Independence}
\begin{proposition}
% (Eq.~\ref{eq:sup_equiv_sde})
% The constructed equivalent SDE (Eq.22) shares the same marginal family $\{p_t\}$ with the RF ODE $dx/dt = v_\theta$, and the ultimately derived editing direction is independent of the specific choice of $g(t)$.

% The constructed equivalent SDE (Eq.~\ref{eq:sup_equiv_sde}) shares the same marginal family $\{p_t\}$ with the RF ODE $dx/dt = v_\theta$, and the ultimately derived editing direction is independent of the specific choice of $g(t)$.

The constructed equivalent SDE (Eq.22) shares the same marginal family $\{p_t\}$ with the RF ODE $dx/dt = v_\theta$, and the ultimately derived editing direction is independent of the specific choice of $g(t)$.

\end{proposition}
\begin{proof}
Writing the Fokker--Planck equation (FPE) for this equivalent SDE yields:
\begin{equation}
\frac{\partial p_t}{\partial t} = -\nabla\!\cdot(f_{eq}\,p_t) + \frac{g^2}{2}\nabla^2 p_t.
\end{equation}
Substituting $f_{eq}$ and expanding, while utilizing the identity $\nabla\!\log p_t \cdot p_t = \nabla p_t$, the second term $\frac{g^2}{2}\nabla^2 p_t$ is exactly canceled out, resulting in:
\begin{equation}
\frac{\partial p_t}{\partial t} = -\nabla\!\cdot(v_\theta\,p_t).
\end{equation}
This is exactly the continuity equation of the RF ODE. Since both satisfy the same partial differential equation with the identical initial condition $p_0$, their marginals at any arbitrary time $t$ must be strictly equal. Furthermore, in the subsequent $h$-transform derivations, the coefficient $g^2/2$ exactly cancels against the temporal factors in the Gaussian score gradients, assuring that the formulated editing velocity field is physically unique and $g$-independent.
\end{proof}
% ============================================================
\section{Full Derivation of the Guidance Terms}
\label{sup:guidance}
% ============================================================
% (Eq.~\ref{eq:h_rec_velocity}) and the editing direction (Eq.~\ref{eq:delta_edit}) 
This section provides the rigorous derivations of the reconstruction guidance (Eq.13) and the editing direction (Eq.16) presented in Sec.~4.2 of the main text, showing how abstract probability constraints are translated into computable velocity vectors.

\subsection{Source Consistency Guidance ($h_{rec}$)}
\label{sup:h_rec}
To enforce the source constraint $\{x_0 = x_0^{src}\}$, we define the terminal condition $h_{rec}(x_0,0) = \delta(x_0 - x_0^{src})$. By the harmonic property of Doob's $h$-transform, the guidance function at time $t$ evaluates to $h_{rec}(x_t,t) = p(x_0^{src} \mid x_t)$. Applying Bayes' rule, the gradient of the log-harmonic function decomposes into the likelihood and the marginal score:
\begin{equation}
\nabla_{x_t}\!\log h_{rec} = \nabla_{x_t}\!\log p(x_t \mid x_0^{src}) - \nabla_{x_t}\!\log p_t(x_t).
\label{eq:sup_bayes_rec}
\end{equation}
 % Eq.~\ref{eq:sup_bayes_rec} 
Under the Rectified Flow transition kernel $p(x_t \mid x_0) = \mathcal{N}\bigl((1{-}t)x_0,\,t^2 I\bigr)$, the likelihood gradient is exactly $(1{-}t)x_0^{src}/t^2 - x_t/t^2$. For the marginal score, we invoke the Tweedie identity and approximate the posterior mean using the source-conditioned velocity $\hat{x}_0 = x_t - t\,v_\theta^{src}$. Substituting these into Eq.25 and multiplying by the RF bridge coefficient $-\frac{t}{1-t}$ yields the precise bridge guidance:
\begin{align}
-\frac{t}{1{-}t}\nabla\!\log h_{rec} &\approx -\frac{t}{1{-}t} \left[ \frac{(1{-}t)\,x_0^{src} - x_t}{t^2} - \frac{(1{-}t)(x_t - t\,v_\theta^{src}) - x_t}{t^2} \right] \notag \\[4pt]
&= \frac{x_t - x_0^{src}}{t} - v_\theta^{src} \notag \\[4pt]
&= v_{rec} - v_\theta^{src}.
\end{align}
Adding this guidance term to the base velocity $v_\theta^{src}$ yields the conditioned velocity $v_{rec} = (x_t - x_0^{src})/t$, which matches the conditional velocity theorem in Flow Matching~\cite{lipman2022flow}, verifying its structural correctness.

\subsection{Editing Guidance ($h_{edit}$)}
\label{sup:h_edit}
For semantic editing, the harmonic property dictates $h_{edit}(x_t,t) = p(c^{edit} \mid x_t)$. Expanding this via Bayes' rule gives:
\begin{equation}
\nabla\!\log h_{edit} = \nabla\!\log p(x_t \mid c^{edit}) - \nabla\!\log p_t(x_t).
\end{equation}
To isolate the pure semantic modification and discard the shared underlying structure, we introduce a deliberate modeling approximation: replacing the unconditional marginal score $\nabla\!\log p_t(x_t)$ with the source-conditioned score $\nabla\!\log p(x_t \mid c^{src})$. This relative gradient formulation effectively captures the vector from the source concept to the target concept. Applying the Tweedie estimate for both conditional scores gives:
\begin{align}
\nabla\!\log \frac{p(x_t \mid c^{edit})}{p(x_t \mid c^{src})} &\approx \frac{(1{-}t)(\hat{x}_0^{edit} - \hat{x}_0^{src})}{t^2} \notag \\[4pt]
&= \frac{(1{-}t)\bigl[(x_t - t\,v_\theta^{edit}) - (x_t - t\,v_\theta^{src})\bigr]}{t^2} \notag \\[4pt]
&= -\frac{1{-}t}{t}(v_\theta^{edit} - v_\theta^{src}).
\end{align}
Multiplying this by the bridge coefficient $-\frac{t}{1-t}$, the temporal factors $(1{-}t)/t$ and $t/(1{-}t)$ exactly cancel out, resulting in a clean, time-independent velocity modification:
\begin{equation}
-\frac{t}{1{-}t} \nabla\!\log \frac{p(x_t \mid c^{edit})}{p(x_t \mid c^{src})} \approx v_\theta^{edit} - v_\theta^{src} = \Delta_{edit}.
\end{equation}

\subsection{Approximation Quality}
\label{sup:approx_quality}
In deriving $\Delta_{edit}$, the approximation error is proportional to $\|\delta v^{edit} - \delta v^{src}\|$, where $\delta v^c = v_\theta(\cdot,c) - v^*(\cdot,c)$ denotes the model's prediction error under condition $c$. Since these two evaluations share identical network parameters and inputs $(x_t, t)$, and the text conditions are highly correlated, their truncation and prediction errors largely cancel each other out during the subtraction. Thus, this relative formulation yields an extremely precise editing direction gradient. In contrast, the reconstruction guidance $v_{rec}$ utilizes the exact ground-truth $x_0^{src}$ via the known Gaussian kernel, making it inherently free from neural network prediction errors.

% ============================================================
\section{Analysis of the Additive Guidance Parameterization}
\label{sup:additive}
% ============================================================
% Eq.~\ref{eq:vbase}
In Eq.15 of the main text, we formulate the reconstruction-guided baseline using an additive parameterization: $\tilde{v}_{rec} = v_\theta^{src} + \lambda_{rec}\,v_{rec}$. This section provides the mathematical and geometric justifications for this specific design over the standard interpolation-based approach.

\subsection{Robustness Against Model Prediction Errors}
To understand the mathematical necessity of our additive design, it is instructive to consider the conventional interpolation-based alternative:
\begin{equation}
\tilde{v}_{rec}^{(\text{interp})} = (1{-}\lambda)\,v_\theta^{src} + \lambda\,v_{rec}, \quad \lambda \in [0,1].
\end{equation}
To demonstrate why our additive form is strictly superior, we decompose the neural network output as $v_\theta^{src} = v^* + \varepsilon$, where $v^*$ is the theoretically optimal conditional velocity and $\varepsilon$ is the unavoidable prediction error. The reconstruction velocity $v_{rec} = (x_t - x_0^{src})/t$ is derived directly from the ground-truth $x_0^{src}$ via the analytical Gaussian kernel, rendering it inherently error-free.

Under the interpolation parameterization, the velocity becomes:
\begin{equation}
\tilde{v}_{rec}^{(\text{interp})} = (1{-}\lambda)v^* + \lambda\,v_{rec} + (1{-}\lambda)\varepsilon.
\end{equation}
While increasing $\lambda$ suppresses the error term $(1{-}\lambda)\varepsilon$, it simultaneously diminishes the true model signal $(1{-}\lambda)v^*$. Consequently, strong guidance ($\lambda \to 1$) collapses the model's generative prior, leading to overly smoothed or corrupted integration trajectories.

In contrast, our additive parameterization yields:
\begin{equation}
\tilde{v}_{rec} = v^* + \lambda_{rec}\,v_{rec} + \varepsilon.
\end{equation}
This formulation preserves the exact optimal signal $v^*$ and the error $\varepsilon$ at a constant scale ($1.0$), while injecting a mathematically pure, error-free reconstruction pull $\lambda_{rec}\,v_{rec}$. This decoupling mathematically guarantees that we can inject an arbitrarily strong structural constraint without eroding the foundational generative capabilities of the pre-trained ODE.

\subsection{Geometric Interpretation via Convex Decomposition}
Directly adding an unbounded term $\lambda_{rec}\,v_{rec}$ to the velocity field might intuitively raise concerns about causing drastic directional distortion or unbounded trajectory divergence. To theoretically dispel this, we demonstrate that our additive form strictly factorizes into a bounded directional interpolation and a scalar magnitude amplification:
\begin{equation}
\tilde{v}_{rec} = \underbrace{(1{+}\lambda_{rec})}_{\text{Magnitude Amplification}} \cdot \underbrace{\left[\frac{1}{1{+}\lambda_{rec}}\,v_\theta^{src} + \frac{\lambda_{rec}}{1{+}\lambda_{rec}}\,v_{rec}\right]}_{\text{Bounded Directional Interpolation}}.
\end{equation}

The bracketed term is a strict convex combination with an effective weight $\alpha = \lambda_{rec}/(1{+}\lambda_{rec}) \in [0,1)$. This reveals a critical geometric property: no matter how large the coefficient $\lambda_{rec}$ becomes, the \emph{direction} of the modified velocity $\tilde{v}_{rec}$ is permanently bounded within the convex cone formed by the model's prior direction ($v_\theta^{src}$) and the exact analytical direction ($v_{rec}$). The prefactor $(1{+}\lambda_{rec})$ merely amplifies the magnitude of the velocity along this safe, interpolated direction. Thus, the additive parameterization geometrically accelerates the convergence towards the source image $x_0^{src}$ without violating the valid directional manifold of the vector field.

% ============================================================
\section{Orthogonal Decomposition: Motivation and Proof}
\label{sup:ortho}
% ============================================================

This section provides the theoretical expansion and formal mathematical proof for the orthogonal decomposition mechanism introduced in Sec.~4.4 of the main text, alongside practical considerations for its algorithmic implementation.

\subsection{Eliminating Fidelity--Editability Entanglement}
As discussed in the main text, the product-of-experts framework linearly superimposes the editing direction $\Delta_{edit}$ onto the reconstruction baseline $\tilde{v}_{rec}$. However, in the high-dimensional vector space of pre-trained flow models, the raw semantic direction $\Delta_{edit}$ naturally retains a parallel component along $\tilde{v}_{rec}$. 

Geometrically, this parallel component $\frac{\langle \Delta_{edit},\, \tilde{v}_{rec}\rangle}{\|\tilde{v}_{rec}\|^2}\tilde{v}_{rec}$ is fundamentally entangled with the structural geometry of the source image. If left unprojected, injecting $\lambda_{edit}\Delta_{edit}$ would force this parallel component to either redundantly reinforce or destructively counteract the reconstruction trajectory. This phenomenon is exactly what causes the notorious \emph{fidelity--editability entanglement} in standard guidance methods: tuning up the editing strength inadvertently shatters the source structure.
% Eq.~\ref{eq:ortho_proj}
By projecting $\Delta_{edit}$ onto the orthogonal complement (Eq.17), we surgically discard this interfering parallel component. The surviving orthogonal component, $\Delta_{edit}^{\perp}$, lies purely in the null space of the reconstruction constraint. Consequently, it carries strictly the \emph{novel semantic content} of $c^{edit}$ that is absent from the source trajectory, allowing the model to hallucinate new concepts without touching the preserved geometric foundation.

\subsection{Mathematical Proof of Strict Decoupling}
We now formally prove the geometric claim made in Sec.~4.4: varying $\lambda_{edit}$ modifies the edit without degrading fidelity, and varying $\lambda_{rec}$ adjusts fidelity without suppressing the edit direction.
% Eq.~\ref{eq:edit_ode}
\begin{proposition}[Strict Subspace Decoupling]
\label{prop:practical_decoupling}
Let $v = \tilde{v}_{rec} + \lambda_{edit}\,\Delta_{edit}^{\perp}$ be the total velocity driving the editing ODE (Eq.19), where $\tilde{v}_{rec} = v_\theta^{src} + \lambda_{rec}\,v_{rec}$ and $\Delta_{edit}^{\perp} \perp \tilde{v}_{rec}$ by construction. Then:
\begin{enumerate}
\item The structural fidelity, governed by the velocity component along $\tilde{v}_{rec}$, is absolutely invariant to any changes in the editing strength $\lambda_{edit}$.
\item Any modification to the reconstruction strength $\lambda_{rec}$ updates the baseline $\tilde{v}_{rec}$, and the projection dynamically adapts to ensure the editing energy remains entirely orthogonal to the new structural constraint.
\end{enumerate}
\end{proposition}

\begin{proof}
For claim (1): The structural fidelity is determined by the scalar projection of the total velocity $v$ along the direction of the reconstruction baseline $\tilde{v}_{rec}$. We evaluate this component mathematically:
\begin{equation}
\frac{\langle v, \tilde{v}_{rec}\rangle}{\|\tilde{v}_{rec}\|} = \frac{\langle \tilde{v}_{rec} + \lambda_{edit}\,\Delta_{edit}^{\perp}, \tilde{v}_{rec}\rangle}{\|\tilde{v}_{rec}\|} = \frac{\|\tilde{v}_{rec}\|^2 + \lambda_{edit} \cdot 0}{\|\tilde{v}_{rec}\|} = \|\tilde{v}_{rec}\|.
\end{equation}
Since this scalar component strictly equals $\|\tilde{v}_{rec}\|$, its partial derivative with respect to $\lambda_{edit}$ is exactly zero. This proves that scaling the target semantics operates on a strictly independent orthogonal subspace, leaving the source fidelity completely unperturbed.
% Algorithm~\ref{alg:editing}
For claim (2): When $\lambda_{rec}$ is adjusted, $\tilde{v}_{rec}$ structurally changes. According to Algorithm 1 (Lines 7--11), the orthogonal direction is dynamically recomputed at every single step as $\Delta_{edit}^{\perp} = \Delta_{edit} - \frac{\langle \Delta_{edit},\, \tilde{v}_{rec}\rangle}{\|\tilde{v}_{rec}\|^2}\,\tilde{v}_{rec}$. By the definition of vector projection, this newly formulated $\Delta_{edit}^{\perp}$ mathematically guarantees $\langle \Delta_{edit}^{\perp}, \tilde{v}_{rec}\rangle = 0$ against the updated baseline, preserving the decoupling indefinitely.
\end{proof}
% Algorithm~\ref{alg:editing}
\subsection{Practical Implementation Details}
This subsection addresses the practical handling of the terms evaluated in Algorithm 1 under discrete sampling schedules. In Line 6 of Algorithm 1, the exact reconstruction velocity $v_{rec} = (x_t - x_0^{src})/t$ involves a division by $t$. While this poses a theoretical singularity at $t = 0$ for continuous-time ODEs, our framework utilizes a discrete schedule (e.g., $N=28$ steps for FLUX). In practical editing pipelines, we perform \emph{partial inversion} and generation, deliberately terminating the trajectory at $t_1 > 0$ (e.g., stopping at step 24 out of 28). This discrete bounding effectively processes the calculation without encountering the $t=0$ limit, completely eliminating the need for heuristic magnitude clipping. Furthermore, for the projection denominator in Line 11 ($\|\tilde{v}_{rec}\|^2$), exact zero-norm cancellation is practically impossible during a valid generation manifold. Nonetheless, to ensure deterministic safety in lower-precision floating-point operations (e.g., FP16 or BF16 commonly used in diffusion inference), we append a standard arithmetic epsilon ($\epsilon = 10^{-8}$) to the denominator.

% ============================================================
%             PART II: EXPERIMENTAL SUPPLEMENTS
% ============================================================

% ============================================================
\section{Extension to Complex Multi-Target Editing}
\label{sup:complex}
% ============================================================

The standard $h$-Flow framework computes a single editing
direction $\Delta_{edit} = v_\theta^{edit} - v_\theta^{src}$
from one target prompt $c^{edit}$ and applies one global
orthogonal projection. This works well for single-target
edits (\eg, changing an object's species or color). For
complex prompts that request multiple simultaneous
modifications, two factors degrade the single-direction
strategy: (i)~current flow models have limited compositional
text understanding, so $v_\theta^{edit}$ from a long,
multi-target prompt may not faithfully encode every requested
change; and (ii)~even when individually captured, the
heterogeneous semantic signals are entangled in one vector
$\Delta_{edit}$, making a single projection unable to
independently control each sub-edit.

\subsection{Decompose-then-Project}

Given a complex prompt involving $K$ distinct editing targets,
we decompose it into $K$ atomic sub-prompts
$\{c_1^{edit}, \ldots, c_K^{edit}\}$, each modifying exactly
one attribute while keeping the rest identical to $c^{src}$.
For each sub-prompt, we compute an independent editing
direction and project it onto the orthogonal complement of
$\tilde{v}_{rec}$:
\begin{align}
\Delta_{edit}^{(k)} &= v_\theta(x_t, t, c_k^{edit})
  - v_\theta^{src}, \\[4pt]
\Delta_{edit}^{(k),\perp} &= \Delta_{edit}^{(k)}
  - \frac{\langle \Delta_{edit}^{(k)},\,
  \tilde{v}_{rec}\rangle}
  {\|\tilde{v}_{rec}\|^2}\,\tilde{v}_{rec}.
\end{align}
Each $\Delta_{edit}^{(k),\perp}$ independently satisfies
$\langle \Delta_{edit}^{(k),\perp},\,\tilde{v}_{rec}
\rangle = 0$, inheriting the fidelity-preserving guarantee
of Proposition~\ref{prop:practical_decoupling}. This
bypasses the compositional understanding bottleneck: each
$v_\theta(x_t, t, c_k^{edit})$ only needs to encode one
simple modification.

\subsection{Weighted Aggregation}

The projected sub-edits are aggregated via:
\begin{equation}
\frac{dx_t}{dt} = \tilde{v}_{rec}
  + \sum_{k=1}^{K} w_k \cdot
  \Delta_{edit}^{(k),\perp},
\label{eq:complex_ode}
\end{equation}
where $w_k \geq 0$ controls each sub-edit's strength.
Since every $\Delta_{edit}^{(k),\perp} \perp
\tilde{v}_{rec}$, the aggregate remains orthogonal to
$\tilde{v}_{rec}$ by linearity of inner products---no
additional projection is needed after summation. We
currently set uniform weights $w_k = \lambda_{edit}/K$,
which recovers standard $h$-Flow when $K{=}1$.

\subsection{Visualization}

% Figure~\ref{fig:complex-projection} compares three
% strategies. Without projection, increasing editing
% strength degrades source structure. The standard global
% projection alleviates this but struggles to fully execute
% all changes in the multi-target case, since the entangled
% $\Delta_{edit}$ loses per-target granularity after one
% projection. The decomposed variant (\emph{h}-Flow-complex)
% successfully applies each modification while preserving
% source fidelity. For single-target editing, both variants
% perform comparably, confirming no degradation from
% decomposition.
Figure~\ref{fig:complex-projection} illustrates the effect of orthogonal projection and its decomposed extension. Comparing the first two columns, the result without projection exhibits noticeable structural degradation in unedited regions, whereas the standard
global projection effectively preserves source fidelity (quantitatively validated in Tab.~5 of the main text).
Comparing the last two columns in the multi-target case (top row), the global projection struggles to fully
execute all requested changes, as the entangled $\Delta_{edit}$ loses per-target granularity after a single projection. The decomposed variant
(\emph{h}-Flow-complex) successfully applies each modification independently. In the single-target case (bottom row), all projection variants perform comparably, confirming no degradation from decomposition.
\begin{figure}[t]
\centering
\includegraphics[width=\linewidth]{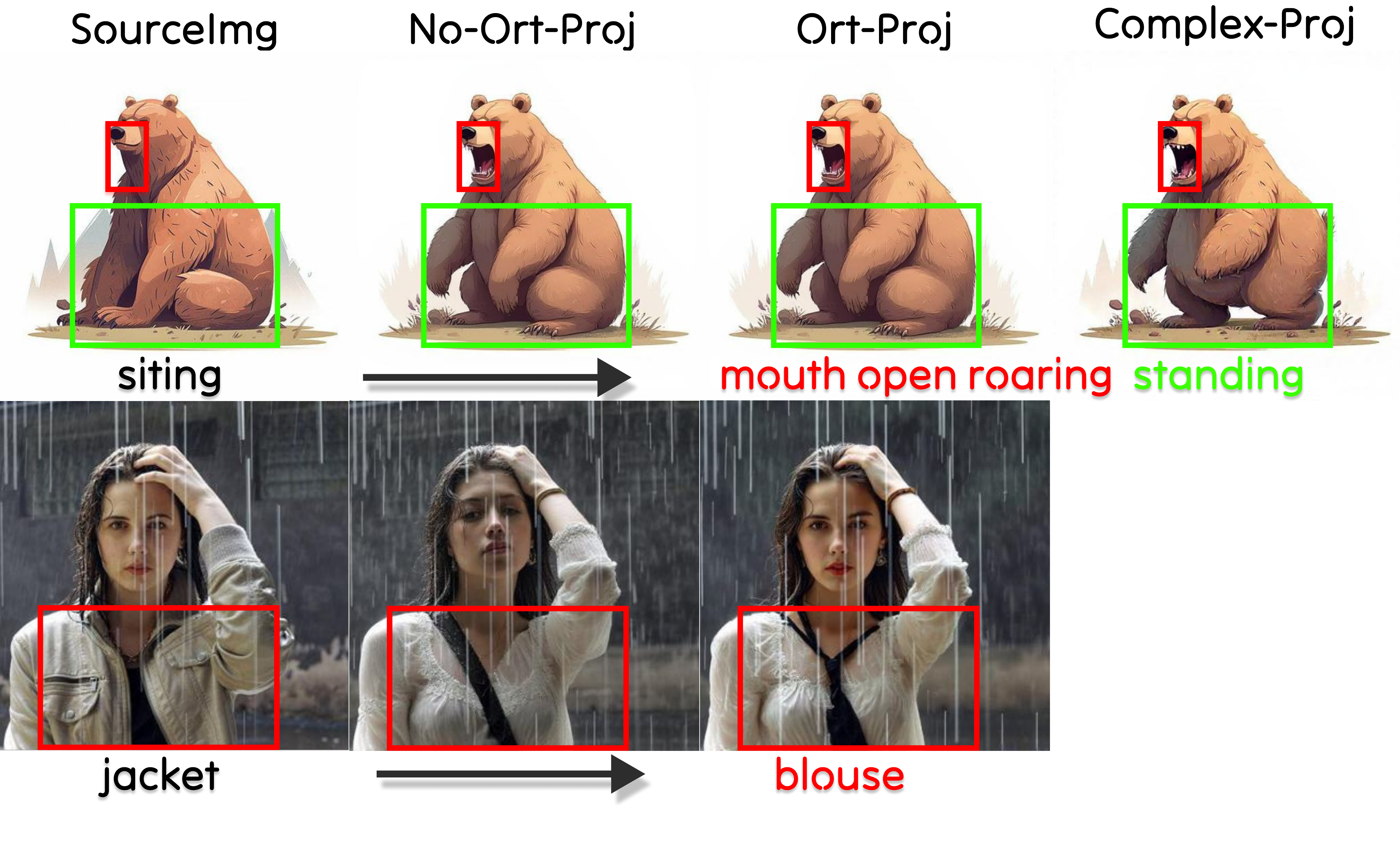}
% \caption{\textbf{Visualization of projection variants.}
% Top: multi-target editing. Bottom: single-target editing.
% Left to right: source, no projection, global projection,
% decomposed per-sub-edit projection (\emph{h}-Flow-complex).}
\caption{\textbf{Visualization of projection variants.} Top row: multi-target editing; Bottom row: single-target editing. ``Complex-Proj" denotes our extended strategy.}
\label{fig:complex-projection}
\end{figure}

\subsection{Discussion}

The uniform weighting treats all sub-edits equally, which
may be suboptimal when they differ in complexity or spatial
extent. Replacing static weights with content-adaptive
functions $w_k(x_t, t)$---\eg, emphasizing global changes
at early timesteps and localized edits at later
timesteps---is a promising direction for future work. The
decomposed variant requires $K{+}1$ forward passes per
step versus 2 for standard $h$-Flow; the projection
overhead ($K$ inner products) remains negligible.

% 新增章节，放在 G（投影可视化）之后、H（定性结果）之前
% G.1 的 stronger inversion 表也移到这里

\section{Additional Comparative Experiments}
\label{sup:additional_comparisons}

To further validate the generality of $h$-Flow, we
conduct additional experiments across different base
models, benchmarks, and inversion backends.

\subsection{Generalization Across Flow Models}
The main text evaluates $h$-Flow on FLUX.1-dev. To
verify that the framework generalizes to other
pre-trained flow models, we additionally evaluate
on Stable Diffusion 3 (SD3).
Tab.~\ref{tab:cross_model} compares $h$-Flow against
recent training-free editing methods. On SD3,
$h$-Flow achieves strong target alignment while
maintaining competitive source consistency,
confirming that the $h$-transform formulation is not
tied to a specific model but applies broadly to
flow matching models with Gaussian transition kernels.

\begin{table}[t]
\centering
\caption{Comparison across flow models on PIE-Bench.}
\label{tab:cross_model}
\resizebox{\linewidth}{!}{
\begin{tabular}{llccccc}
\toprule
\multirow{2}{*}{Method} & \multirow{2}{*}{Model}
  & \multicolumn{3}{c}{\textbf{Source Consistency}}
  & \multicolumn{2}{c}{\textbf{Target Alignment}} \\
\cmidrule(lr){3-5} \cmidrule(lr){6-7}
& & LPIPS$\downarrow$ & MSE$\downarrow$ & SSIM$\uparrow$
& CLIP\_T$\uparrow$ & CLIP\_E$\uparrow$ \\
\midrule
DNAEdit~\cite{xie2025dnaedit}  & FLUX & 0.087 & 0.005 & \textbf{0.866}
         & 24.52 & 22.06 \\
$h$-Flow & FLUX & 0.110 & 0.006 & 0.840
         & \textbf{25.80} & \textbf{22.80} \\
\midrule
FIA-Edit~\cite{yang2026fia} & SD3.5 & 0.0956 & 0.0059 & \textbf{0.868}
         & 25.55 & 22.51 \\
FTEdit~\cite{xu2025unveil}   & SD3.5 & \textbf{0.0894} & 0.0063 & 0.846
         & 24.62 & 22.66 \\
DRFS~\cite{beaudouin2026delta}     & SD3   & 0.0941 & 0.0068 & 0.846
         & 25.60 & 22.72 \\
$h$-Flow & SD3   & 0.086 & \textbf{0.0060} & 0.860
         & \textbf{25.96} & \textbf{23.04} \\
\bottomrule
\end{tabular}}
\end{table}

\subsection{Evaluation on DNA-Bench}
To assess robustness beyond PIE-Bench, we evaluate on
DNA-Bench~\cite{xie2025dnaedit}, a benchmark featuring
diverse and natural editing scenarios with detailed textual
descriptions. Since these prompts tend to be longer and
more descriptive than PIE-Bench, we apply the
decompose-then-project strategy
(Section~\ref{sup:complex}) to better handle the rich
textual content. Tab.~\ref{tab:dna_bench} shows that
$h$-Flow achieves comparable source consistency to
DNAEdit while delivering stronger target alignment.

\begin{table}[t]
\centering
\caption{Comparison on DNA-Bench.}
\label{tab:dna_bench}
\resizebox{\linewidth}{!}{
\begin{tabular}{llccccc}
\toprule
\multirow{2}{*}{Method} & \multirow{2}{*}{Model}
  & \multicolumn{3}{c}{\textbf{Source Consistency}}
  & \multicolumn{2}{c}{\textbf{Target Alignment}} \\
\cmidrule(lr){3-5} \cmidrule(lr){6-7}
& & LPIPS$\downarrow$ & MSE$\downarrow$ & SSIM$\uparrow$
& CLIP\_T$\uparrow$ & CLIP\_E$\uparrow$ \\
\midrule
DNAEdit~\cite{xie2025dnaedit}  & FLUX & \textbf{0.08} & \textbf{0.0047}
         & \textbf{0.87} & 27.64 & 22.50 \\
$h$-Flow & FLUX & \textbf{0.08} & 0.0055
         & \textbf{0.87} & \textbf{27.85} & \textbf{22.85} \\
\bottomrule
\end{tabular}}
\end{table}

\subsection{Extended Comparison on PIE-Bench++}

Tab.~\ref{tab:piebenchpp_extended} evaluates $h$-Flow
and its decomposed variant ($h$-Flow-complex,
Section~\ref{sup:complex}) on PIE-Bench++ against
SplitFlow~\cite{yoon2025splitflow}, a method
specifically designed for multi-target editing.
$h$-Flow-complex surpasses both the standard $h$-Flow
and SplitFlow in target alignment (CLIP\_T and CLIP\_E),
demonstrating that the decompose-then-project strategy
effectively handles complex multi-attribute prompts.

\begin{table}[t]
\centering
\caption{Extended comparison on PIE-Bench++.}
\label{tab:piebenchpp_extended}
\resizebox{\linewidth}{!}{
\begin{tabular}{llccccc}
\toprule
\multirow{2}{*}{Method} & \multirow{2}{*}{Model}
  & \multicolumn{3}{c}{\textbf{Source Consistency}}
  & \multicolumn{2}{c}{\textbf{Target Alignment}} \\
\cmidrule(lr){3-5} \cmidrule(lr){6-7}
& & LPIPS$\downarrow$ & MSE$\downarrow$ & SSIM$\uparrow$
& CLIP\_T$\uparrow$ & CLIP\_E$\uparrow$ \\
\midrule
SplitFlow~\cite{yoon2025splitflow}        & SD3 & \textbf{0.0896} & 0.0080
                 & 0.8559 & 26.05 & 25.17 \\
$h$-Flow         & SD3 & 0.1036 & \textbf{0.0056}
                 & \textbf{0.8570} & 25.50 & 24.80 \\
$h$-Flow-complex & SD3 & 0.1116 & 0.0080
                 & 0.8518 & \textbf{26.23} & \textbf{25.55} \\
\bottomrule
\end{tabular}}
\end{table}

\subsection{Compatibility with Stronger Inversion Methods}

Tab.~\ref{tab:strong_inv} extends the forward
compatibility study (Tab.~2 in the main text) to three
additional inversion methods. Since $h$-Flow provides
its own editing mechanism via the $h$-transform, we use
only the \emph{inversion stage} of each method to obtain
$x_{t_N}$, disabling their method-specific editing
components (\eg, attention injection or feature
replacement). $h$-Flow consistently improves all metrics
across all backends, confirming its forward-agnostic
design.

\begin{table}[t]
\centering
\caption{With stronger inversion on PIE-Bench.}
\label{tab:strong_inv}
\resizebox{\linewidth}{!}{
\begin{tabular}{llccccc}
\toprule
\multirow{2}{*}{Method} & \multirow{2}{*}{Model}
  & \multicolumn{3}{c}{\textbf{Source Consistency}}
  & \multicolumn{2}{c}{\textbf{Target Alignment}} \\
\cmidrule(lr){3-5} \cmidrule(lr){6-7}
& & LPIPS$\downarrow$ & MSE$\downarrow$ & SSIM$\uparrow$
& CLIP\_T$\uparrow$ & CLIP\_E$\uparrow$ \\
\midrule
RF-Solver-Inv~\cite{wang2025taming}    & FLUX & 0.25 & 0.034 & 0.67
                 & 25.23 & 22.90 \\
\quad+$h$-Flow   & FLUX & \textbf{0.23} & \textbf{0.029}
                 & \textbf{0.70} & \textbf{25.73}
                 & \textbf{23.25} \\
\midrule
FireFlow-Inv~\cite{dengfireflow}     & FLUX & 0.24 & 0.031 & 0.67
                 & 25.18 & 22.84 \\
\quad+$h$-Flow   & FLUX & \textbf{0.22} & \textbf{0.027}
                 & \textbf{0.71} & \textbf{25.98}
                 & \textbf{23.35} \\
\midrule
FTEdit-Inv~\cite{xu2025unveil}       & FLUX & 0.25 & 0.036 & 0.69
                 & 25.22 & 22.78 \\
\quad+$h$-Flow   & FLUX & \textbf{0.22} & \textbf{0.027}
                 & \textbf{0.72} & \textbf{25.58}
                 & \textbf{23.10} \\
\bottomrule
\end{tabular}}
\end{table}

% ============================================================

\begin{figure*}[!t]
    \centering
    \includegraphics[width=\textwidth,height=0.8\textheight,keepaspectratio]{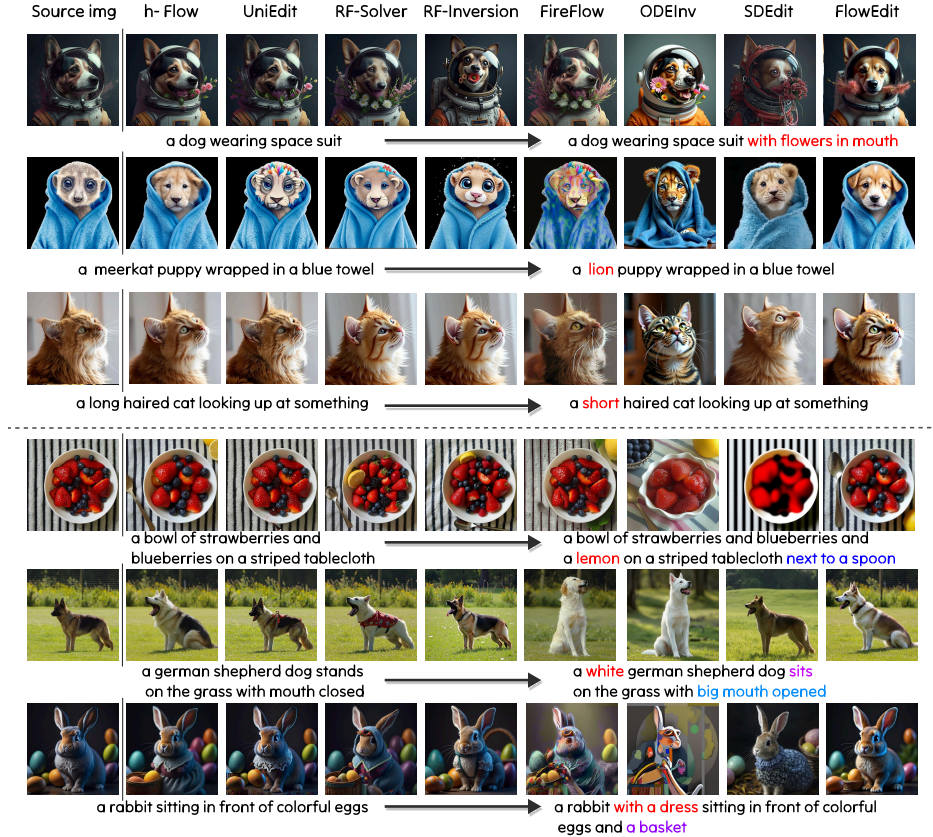}
    \caption{\textbf{Extended Qualitative Comparisons.} We present additional editing results across diverse transformation types. The source and target prompts are displayed below each example. Compared to state-of-the-art baselines, our $h$-Flow consistently achieves the most precise semantic updates while strictly preserving the unedited source content (e.g., the space suit, the blue towel) and background layouts.}
    \label{fig:sup_qualitative}
\end{figure*}

% ============================================================
\section{Extended Qualitative Results}
\label{sup:extended_qualitative}
% ============================================================

In this section, we provide extended qualitative comparisons to further demonstrate the superiority of $h$-Flow across various text-guided image editing scenarios, including subject replacement and object addition.

% \begin{figure*}[htbp]
%     \centering
%     \includegraphics[width=\textwidth]{pic/supplement-compare.pdf} 
%     \caption{\textbf{Extended Qualitative Comparisons.} We present additional editing results across diverse transformation types. The source and target prompts are displayed below each example. Compared to state-of-the-art baselines, our $h$-Flow consistently achieves the most precise semantic updates while strictly preserving the unedited source content (e.g., the space suit, the blue towel) and background layouts.}
%     \label{fig:sup_qualitative}
% \end{figure*}

As demonstrated in Fig.~\ref{fig:sup_qualitative}, baseline methods frequently struggle with the inherent fidelity--editability trade-off across both standard and complex multi-prompt editing datasets. We highlight three typical failure modes in existing approaches that $h$-Flow effectively mitigates:

\noindent\textbf{1. Disrupted Source Consistency and Ineffective Editing:}  When introducing new concepts (e.g., adding "flowers in mouth" to the "dog wearing space suit"), baseline methods frequently struggle to control the editing strength. For instance, FlowEdit applies an excessively aggressive edit, completely destroying the source consistency. The introduction of the new textual condition inadvertently alters the original subject's identity and the underlying geometry of the space suit. Conversely, in localized attribute editing tasks (e.g., changing a "long haired cat" to a "short haired cat"), methods such as UniEdit fail entirely to execute the required modification. In contrast, $h$-Flow seamlessly inserts new concepts and manipulates target attributes while keeping the original subject's identity strictly consistent with the source.

\noindent\textbf{2. Background Distortion during Subject Replacement:} 
In the task of morphing the "meerkat" into a "lion", methods like ODEInv and SDEdit completely fail to retain the source image structure, regenerating an entirely different scene. Other flow-based methods (e.g., UniEdit and RF-Inversion) manage to keep the general layout but drastically alter the complex folds and shading of the "blue towel". While cross-species morphing inherently challenges strict structural preservation, our orthogonal projection mechanism concentrates the editing energy primarily on the subject's facial features. This significantly mitigates background distortion and largely maintains the structural integrity of the towel, achieving a substantially better fidelity-editability balance than existing baselines.

\noindent\textbf{3. Attribute Omission under Complex Multi-Prompt Edits:} 
The bottom three rows feature a more challenging dataset requiring simultaneous multi-attribute and compositional modifications. Under these complex prompts, baseline methods frequently struggle to satisfy all textual conditions simultaneously. For instance, when editing the "german shepherd dog"  to be "white", "sits", and have a "big mouth opened" , methods like UniEdit fail to achieve these multiple editing goals at once, omitting key target attributes. Similarly, in the final row, while FlowEdit successfully maintains the background layout, it completely fails to make the rabbit wear a "dress". $h$-Flow uniquely handles these highly entangled spatial and semantic modifications simultaneously. It successfully morphs the subject's posture and incorporates new concepts, all while effectively anchoring the background layout and environmental context.

\begin{figure*}[htbp]
    \centering
    % 请确保文件名和路径与你本地一致
    \includegraphics[width=\textwidth]{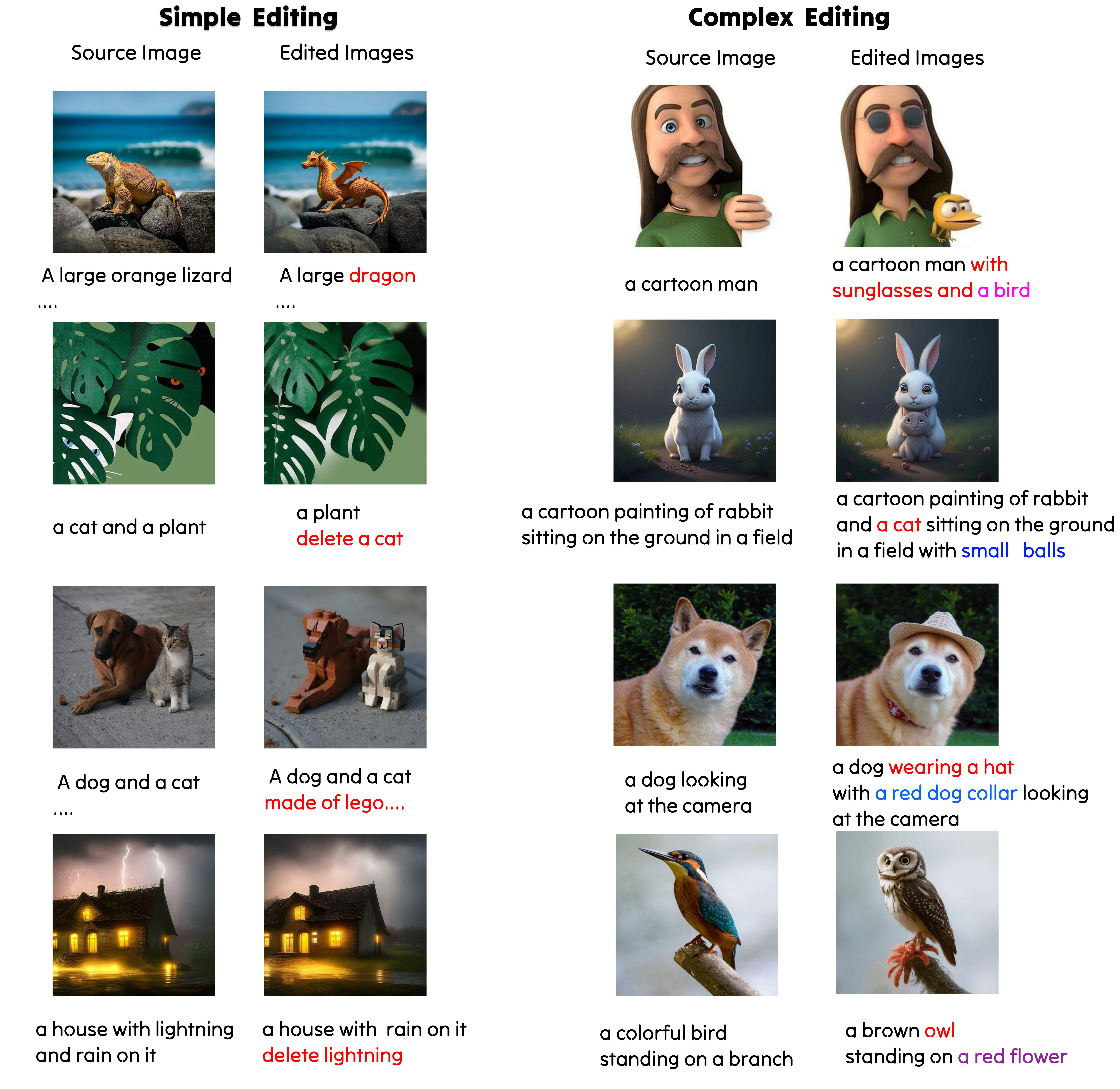} 
    \caption{\textbf{Versatility of $h$-Flow.} We demonstrate the capability of our method across various editing operations—including replacement, addition, deletion, and material modification—under both simple and complex prompt conditions. In all scenarios, $h$-Flow flawlessly executes the precise semantic transformations while rigorously protecting the integrity of the unedited source content.}
    \label{fig:sup_versatility}
\end{figure*}

% \begin{figure*}[htbp]
%     \centering
%     % 请确保文件名和路径与你本地一致
%     \includegraphics[width=\textwidth]{pic/main-pic-2.pdf} 
%     \caption{\textbf{Versatility of $h$-Flow.} We demonstrate the capability of our method across various editing operations—including replacement, addition, deletion, and material modification—under both simple and complex prompt conditions. In all scenarios, $h$-Flow flawlessly executes the precise semantic transformations while rigorously protecting the integrity of the unedited source content.}
%     \label{fig:sup_versatility}
% \end{figure*}

% \subsection{Versatility Across Editing Operations and Prompt Complexities}
% \label{sup:versatility}

To further illustrate the broad applicability and robustness of $h$-Flow, we present additional results encompassing a wide spectrum of editing operations and prompt complexities, as shown in Fig.~\ref{fig:sup_versatility}.

% ============================================================
\section{Limitations and Failure Cases}
\label{sup:limitations}
% ============================================================
While $h$-Flow demonstrates robust disentanglement and fidelity across diverse editing scenarios, it inherits certain limitations fundamentally tied to its reliance on the forward inversion trajectory and the strict mathematical boundaries of orthogonal projection. As illustrated in Fig.~\ref{fig:failure_cases}, we categorize these 8 challenging cases into six typical failure modes:

\begin{itemize}
    \item \textbf{Morphological Distortion (Fig.~\ref{fig:failure_cases}a):} When a target concept inherently necessitates a massive topological departure from the source (e.g., morphing round "fishes" into elongated "sharks", the strict geometric preservation anchored by $\tilde{v}_{rec}$ becomes an excessive constraint. The algorithm forcibly projects the new semantic energy into the original bounding boxes, resulting in unnaturally squeezed subjects.
    
    \item \textbf{Layout Hallucination (Fig.~\ref{fig:failure_cases}b, f):} If a target prompt (e.g., changing a "garland" to a "hat", or a "phone" to "coffee") triggers a dominantly strong, highly entangled unconditioned prior in the base model, it can overpower the reconstruction guidance, synthesizing unexpected layouts, altered outfits, or additional characters.
    
    \item \textbf{Identity Shift via Inversion Loss (Fig.~\ref{fig:failure_cases}e):} Since the reconstruction guidance depends directly on the numerical quality of the forward inversion trajectory, any high-frequency details lost during the forward noise-addition phase (e.g., the specific classical painting identity of the "green dress" portrait) are irreparably replaced by the model's default modern priors.
    
    \item \textbf{Under-editing in Global Shifts (Fig.~\ref{fig:failure_cases}c, g):} For edits involving global illumination or contextual shifts (e.g., transitioning a "bright sky" to a "dark sky", or Earth from "dimly" to "sunlit"), the raw edit vector exhibits exceptionally high cosine similarity with the reconstruction baseline. The orthogonal projection misinterprets these lighting changes as geometry-altering components and truncates them, leading to severe under-editing.
    
    \item \textbf{Structural Persistence (Fig.~\ref{fig:failure_cases}h):} When an edit rigorously requires the spatial removal and relocation of a major structural component (e.g., changing "hands hanging down" into "hands up"), the reconstruction prior strongly anchors the original limbs. The orthogonal component lacks the unconstrained energy to synthesize entirely new limbs while erasing old ones, leading to pose transition failures and suppressing coupled edits (e.g., "black shirt" to "white shirt").
    
    \item \textbf{Attribute Omission (Fig.~\ref{fig:failure_cases}d):} In multi-concept editing tasks, high-contrast structural priors can occasionally repel localized material updates. When transforming "meat balls" into "Tuna sushi", the strong specular highlights of the original ceramic plate create a rigid structural anchor that inadvertently suppresses the weaker semantic signal of the "metal plate" attribute.
\end{itemize}

\begin{figure*}[htbp]
    \centering
    % 请确保文件名与您本地上传的 8 张图拼合的 PDF 一致
    \includegraphics[width=\textwidth]{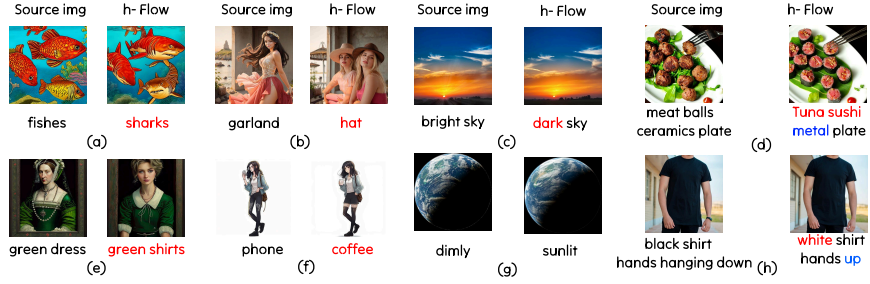} 
    \caption{\textbf{Typical Failure Cases of $h$-Flow.} We identify six failure modes constrained by inversion dependency and rigid orthogonal projection across 8 challenging scenarios: (a) Morphological distortion due to topological conflicts; (b, f) Layout hallucination driven by overwhelming target priors; (e) Identity shift caused by forward inversion information loss; (c, g) Under-editing (over-truncation) in global illumination shifts; (h) Structural persistence in extreme articulation changes; (d) Partial attribute omission in multi-target material edits.}
    \label{fig:failure_cases}
\end{figure*}

% \bibliography{main}

\end{document}